\documentclass{article} 
\usepackage{iclr2021_conference,times}

\usepackage{amsmath,amsfonts,bm}

\def\eqref#1{equation~\ref{#1}}

\def\1{\bm{1}}

\DeclareMathAlphabet{\mathsfit}{\encodingdefault}{\sfdefault}{m}{sl}
\SetMathAlphabet{\mathsfit}{bold}{\encodingdefault}{\sfdefault}{bx}{n}

\DeclareMathOperator*{\argmin}{arg\,min}

\usepackage{hyperref}
\usepackage{url}

\usepackage{algorithmic}
\usepackage{graphicx}
\usepackage{textcomp}
\usepackage{xcolor}
\usepackage{float}
\def\BibTeX{{\rm B\kern-.05em{\sc i\kern-.025em b}\kern-.08em
    T\kern-.1667em\lower.7ex\hbox{E}\kern-.125emX}}
\usepackage{amssymb}
\usepackage{amsthm}

\usepackage[ruled,vlined,linesnumbered]{algorithm2e}
\usepackage{textcomp}
\usepackage{xcolor}
\usepackage{float}

\usepackage{tablefootnote} 
\usepackage{multirow}
\usepackage{multicol}
\usepackage{enumerate}
\usepackage{xspace}
\usepackage[center]{subfigure} 
\usepackage[shortlabels]{enumitem}
\usepackage{booktabs}
\setlist[enumerate]{nosep}
\usepackage{flushend}
\usepackage{mathrsfs}
\usepackage[mathscr]{euscript} 
\usepackage{mdwlist}
\usepackage{adjustbox}

\newcommand{\subfloat}{\subfigure}

\newcommand{\T}[1]{\boldsymbol{\mathscr{#1}}}
\newcommand*{\QEDB}{\hfill\ensuremath{\Box}}%
\newtheorem{theorem}{Theorem}

\newcommand{\method}{\textsc{Falcon}\xspace}
\newcommand{\tucker}{{Tucker}\xspace}
\newcommand{\cp}{{CP}\xspace}

\newcommand{\hide}[1]{}
\newcommand{\gep}{{GEP}\xspace}

\newtheorem{definition}{Definition}

\makeatletter
\newcommand{\printfnsymbol}[1]{%
  \textsuperscript{\@fnsymbol{#1}}%
}
\makeatother

\title{\method: Lightweight and Accurate Convolution}

\author{Jun-Gi Jang$^{1}$\thanks{These authors contributed equally.}, Chun Quan$^{2}$\printfnsymbol{1}, Hyun Dong Lee$^{3}$, U Kang$^{1}$  \\
1 Department of Computer Science and Engineering, Seoul National University \\
2 CCB Fintech \\
3 Department of Computer Science, Columbia University 
}

\iclrfinalcopy 
\begin{document}

\maketitle

\begin{abstract}
	How can we efficiently compress Convolutional Neural Network (CNN) while retaining their accuracy on classification tasks?
Depthwise Separable Convolution (DSConv), which replaces a standard convolution with a depthwise convolution and a pointwise convolution, has been used for building lightweight architectures.
However, previous works based on depthwise separable convolution are limited when compressing a trained CNN model since
1) they are mostly heuristic approaches without a precise understanding of their relations to standard convolution, and
2) their accuracies do not match that of the standard convolution.

In this paper, we propose \method, an accurate and lightweight method to compress CNN.
\method uses GEP, our proposed mathematical formulation to approximate the standard convolution kernel, to interpret existing convolution methods based on depthwise separable convolution.
By exploiting the knowledge of a trained standard model and carefully determining the order of depthwise separable convolution via GEP, \method achieves sufficient accuracy close to that of the trained standard model.
Furthermore, this interpretation leads to developing a generalized version rank-$k$ \method which performs $k$ independent \method operations and sums up the result.
Experiments show that \method 1) provides higher accuracy than existing methods based on depthwise separable convolution and tensor decomposition, and 2) reduces the number of parameters and FLOPs of standard convolution by up to a factor of $8$ while ensuring similar accuracy.
We also demonstrate that rank-$k$ \method further improves the accuracy while sacrificing a bit of compression and computation reduction rates.
\end{abstract}

\section{Introduction}

How can we efficiently reduce the number of parameters and FLOPS of convolutional neural networks (CNN) while maintaining their accuracy on classification tasks?
Nowadays, CNN is widely used in various areas including recommendation system (\cite{convMF}), computer vision (\cite{AlexNet,VGG,inception-v4}), natural language processing (\cite{ABCNN}), etc.
%
Due to an increase in the model capacity of CNNs,
there {has been a considerable interest in building lightweight CNNs.}
A major research direction utilizes depthwise separable convolution (DSConv) (\cite{DSCONV}), which consists of a depthwise convolution and a pointwise convolution, instead of standard convolution.
The depthwise convolution applies a separate 2D convolution kernel to each input channel, and the pointwise convolution changes the channel size using 1$\times$1 convolution.
The depthwise convolution extracts spatial features, and the pointwise convolution merges features along the channel dimension, while a standard convolution simultaneously performs the two tasks. Decoupling the tasks preserves the representation power of the network, and reduces the number of parameters and FLOPS.
Several recent methods (\cite{MobileNet,MobileNetV2,ShuffleNet,ShuffleNetV2}) have successfully built lightweight architectures by 1) constructing CNN with DSConv, and 2) training the model from scratch.


However, existing methods based on depthwise separable convolution cannot effectively compress trained CNNs. They fail to precisely formulate the relation between standard convolution and depthwise separable convolution, leading to the following limitations.
First, it is difficult to utilize the knowledge stored in the trained standard convolution kernels to fit the depthwise separable convolution kernels; i.e.,
existing approaches based on depthwise separable convolution need to be trained from scratch.
Second, although existing methods may give reasonable compression and computation reduction, there is a gap between the resulting accuracy and that of a standard-convolution-based model.
Finally,
generalizing the methods based on depthwise separable convolution is difficult due to their heuristic nature.



In this paper, we propose \method, an accurate and lightweight method to compress CNN by leveraging depthwise separable convolution.
We first precisely define the relationship between the standard convolution and the
depthwise separable convolution using GEP (Generalized Elementwise Product),
our proposed
mathematical formulation to approximate a standard convolution kernel with a depthwise convolution kernel and a pointwise convolution kernel.
GEP allows \method to be fit by a trained standard model, leading to a better accuracy.
\method minimizes the gap of the accuracy between the compressed model with \method and the  standard-convolution-based model by carefully aligning depthwise and pointwise convolutions via GEP.
Based on the precise definition, we generalize \method
to design rank-$k$ \method, which performs $k$ independent \method operations and sums up the result.
We also propose a variant of \method, \method-branch, by integrating \method with the channel split technique (\cite{ShuffleNetV2}), which splits feature channels into two branches, independently processes each branch, and concatenates the results from the branches.
As a result, \method and \method-branch provide a superior accuracy compared to other methods based on {depthwise separable convolution and tensor decomposition,} with similar compression (see Figure~\ref{fig:acc_size}) and the computation reduction rates (see Figure~\ref{fig:acc_flops}).
Rank-$k$ \method further improves accuracy, outperforming even the original convolution in many cases (see Table~\ref{tab:rank_perf}).

\vspace{1mm}

Our contributions are summarized as follows:

\begin{figure*} [t]
	\centering
	\subfloat{\includegraphics[width=0.7\textwidth]{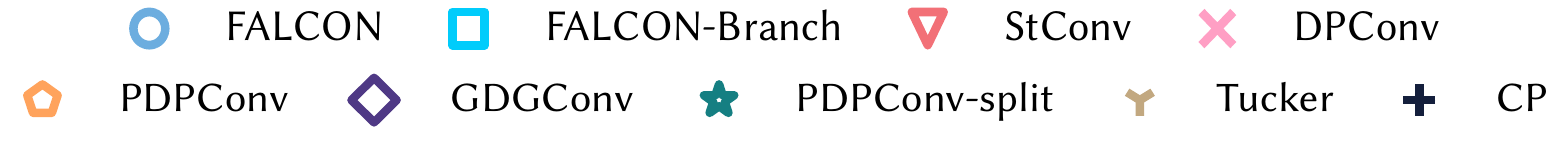}}\vspace{-4mm} \\
	\setcounter{subfigure}{0}
	\subfloat[VGG19-CIFAR10] {\includegraphics[width=0.27\textwidth]{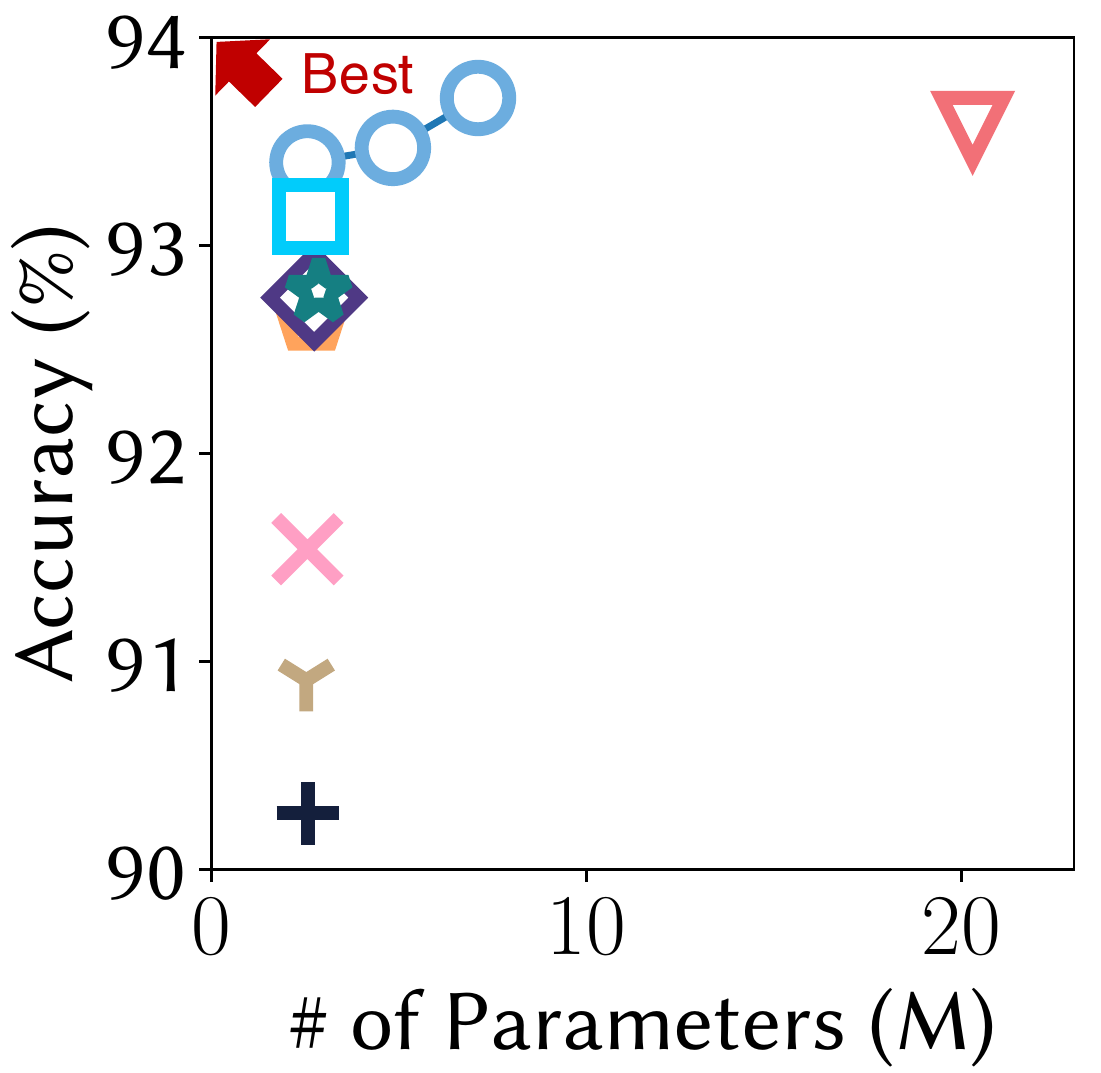}\label{fig:vgg19_cifar10}}
	\subfloat[VGG19-SVHN] {\includegraphics[width=0.265\textwidth]{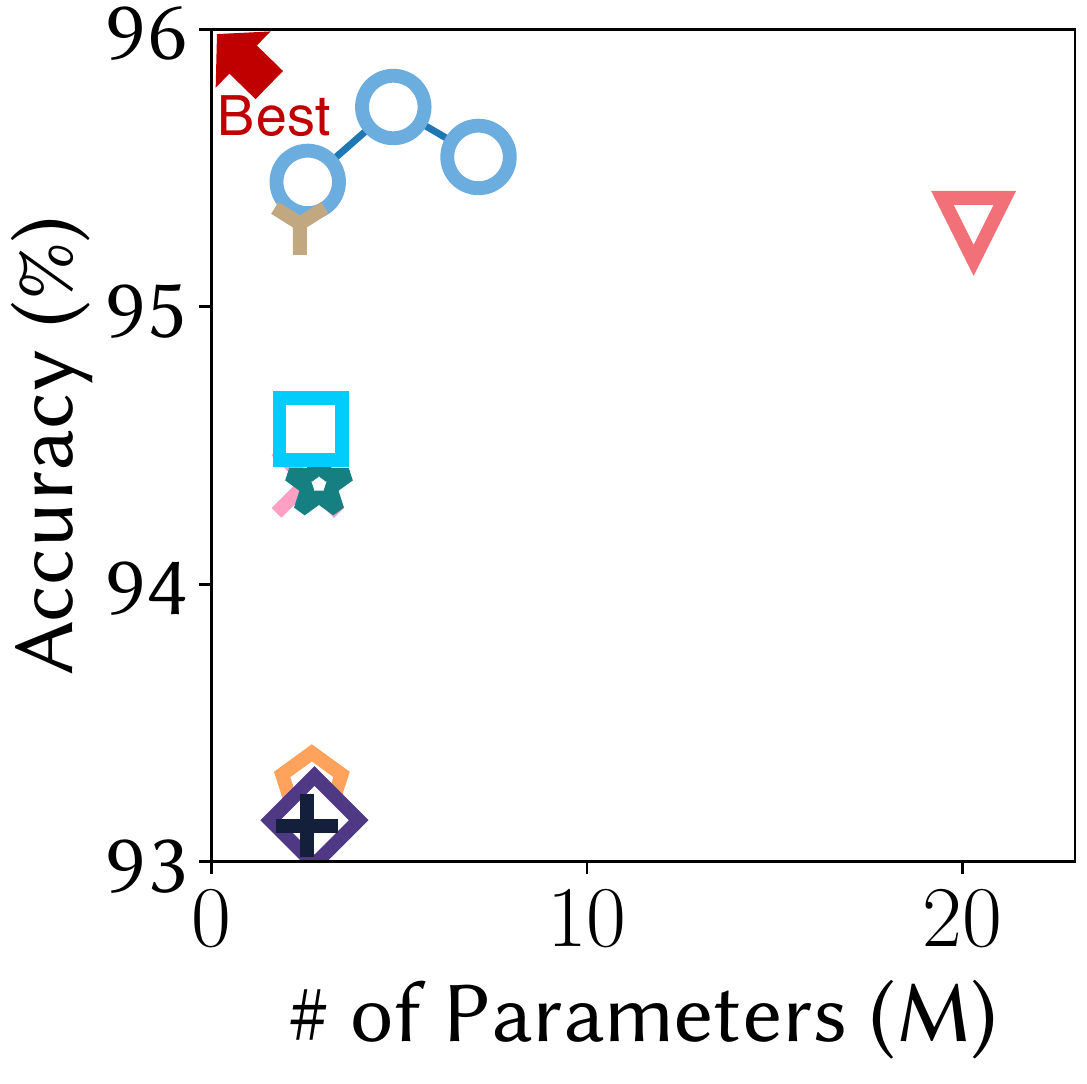}\label{fig:vgg19_svhn}}
	\subfloat[VGG16-ImageNet] {\includegraphics[width=0.28\textwidth]{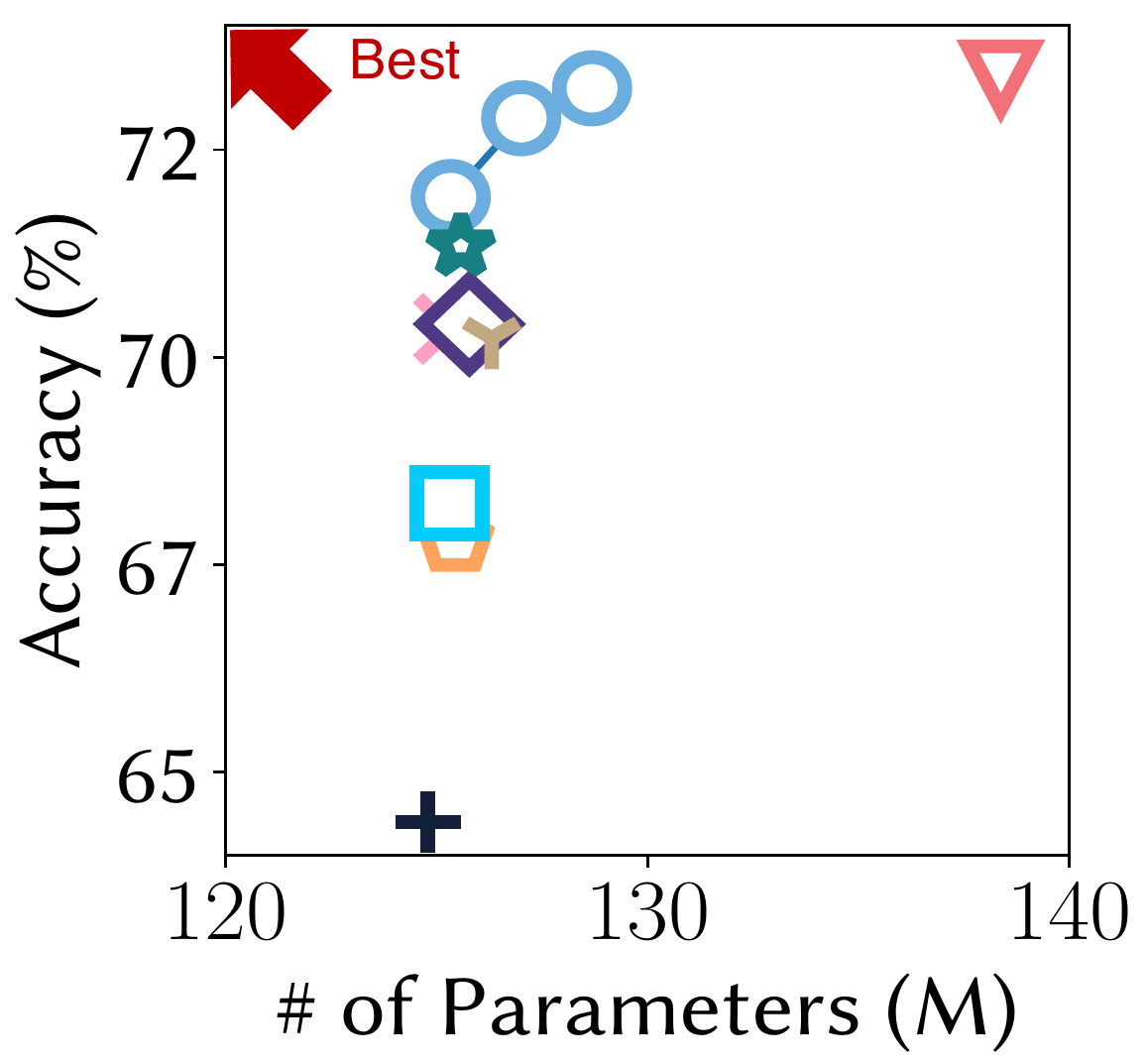}\label{fig:vgg19_imagenet}} \\
	\subfloat[ResNet34-CIFAR10] {\includegraphics[width=0.27\textwidth]{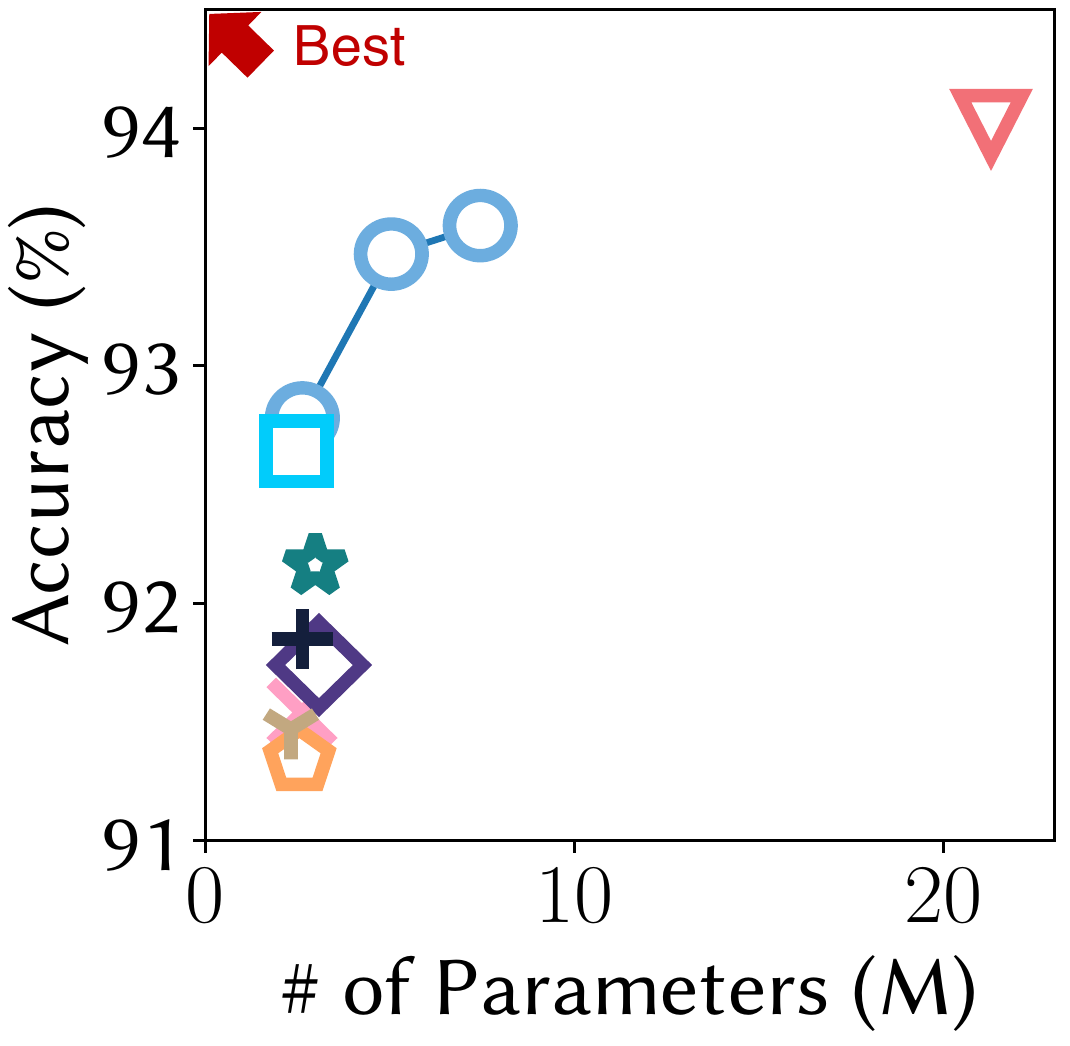}\label{fig:resnet_cifar10}}
	\subfloat[ResNet34-SVHN] {\includegraphics[width=0.275\textwidth]{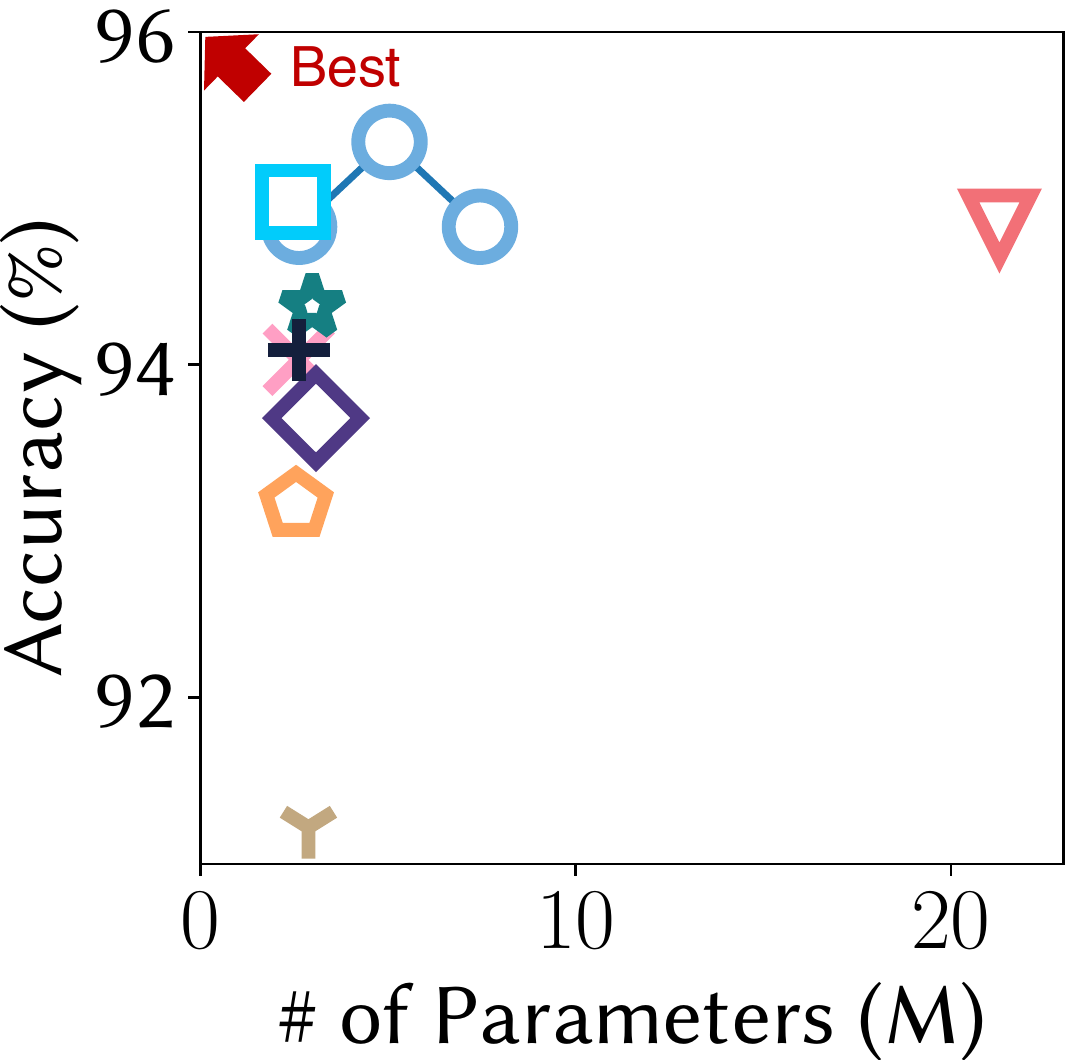}\label{fig:resnet_svhn}}
	\subfloat[ResNet18-ImageNet] {\includegraphics[width=0.27\textwidth]{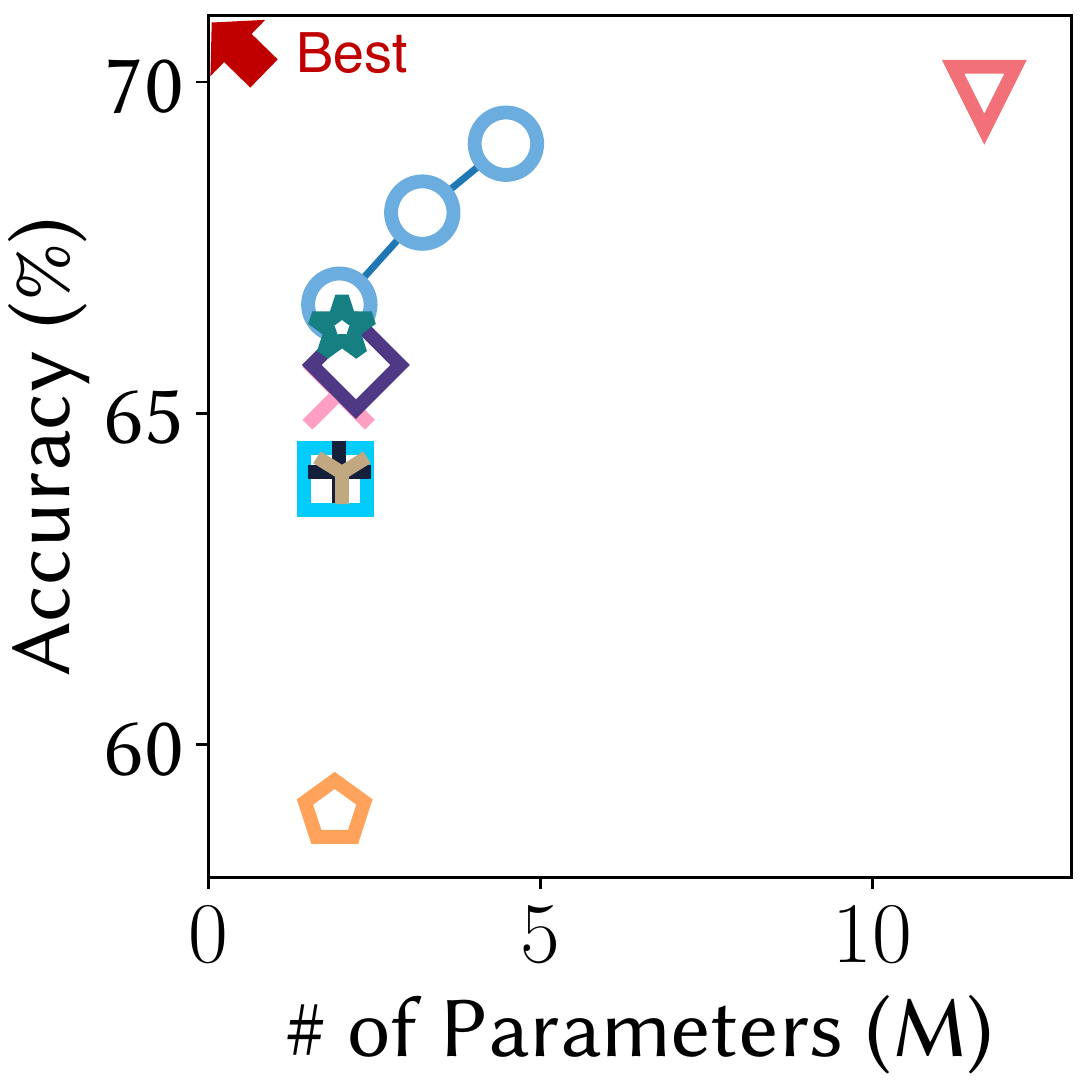}\label{fig:resnet_imagenet}}
	\caption{Accuracy w.r.t. number of parameters from different models and datasets.
The three blue circles correspond to rank-1, 2, and 3 \method (from left to right order), respectively.
\method provides the best accuracy for a given number of parameters.
}
\label{fig:acc_size}
\end{figure*}

\begin{itemize}
%
\item \textbf{Algorithm.} We propose {\method}, a CNN compression method based on depthwise separable convolution.
    \method and its variant are carefully designed to compress CNN with little accuracy loss.
    We also give theoretical analysis of compression and computation reduction rates of \method.
\item \textbf{Generalization.}
We generalize \method to design rank-$k$ \method, which further improves the accuracy, and often outperforms standard convolution, with a little sacrifice in compression and computation reduction rates.
\item \textbf{Experiments.} Extensive experiments show that \method
1) outperforms other state-of-the-art methods based on depthwise separable convolution and tensor decompositions for compressing CNN (see Figure~\ref{fig:acc_size}),
and
2) provides up to $8.61\times$ compression and $8.44\times$ computation reduction compared to the standard convolution while giving similar accuracies.
\end{itemize}

%


%
The code of \method is at \url{https://github.com/snudm-starlab/cnn-compression/tree/master/FALCON2}.
The rest of this paper is organized as follows.
We describe preliminaries,
our proposed method {\method}, and
experimental results.
After discussing related works,
we present conclusion.

\section{Preliminary}

We describe preliminaries on depthwise separable convolution, and methods based on it.
We use the symbols listed in Table~\ref{tab:symbol}.

\subsection{Convolutional Neural Network}
Convolution Neural Network (CNN) is a type of deep neural network used mainly for structured data.
CNN uses convolution operation in convolution layers.
In the following, we discuss CNN when applied to typical image data with RGB channels.

{
Each convolution layer has three components: input feature maps, convolution kernel, and output feature maps.
The input feature maps $\T{I} \in \mathbb{R}^{H \times W \times M}$ and the output feature maps $\T{O} \in \mathbb{R}^{H' \times W' \times N}$ are 3-dimensional tensors, and
the convolution kernel $\T{K} \in \mathbb{R}^{D \times D \times M \times N}$ is a 4-dimensional tensor.
}

The convolution operation is defined as:

\begin{align}
\T{O}_{h',w',n} = \sum^{D}_{i=1} \sum^{D}_{j=1} \sum^{M}_{m=1} \T{K}_{i,j,m,n} \cdot \T{I}_{h_i,w_j,m}
\label{eq:standardconv}
\end{align}
where the relations between height $h_i$ and width $w_j$ of input, and height $h'$ and width $w'$ of output are as follows:

\begin{align}
\label{eq:heightandwidth}
h_i=(h'-1)s+i-p \hspace{5mm}\text{and}\hspace{5mm} w_j=(w'-1)s+j-p
\end{align}
where $s$ is the stride size, and $p$ is the padding size.
The third and the fourth dimensions of the convolution kernel $\T{K}$ must match the number $M$ of input channels, and the number $N$ of output channels, respectively.

Convolution kernel $\T{K}$ can be seen as $N$ 3-dimensional filters $\T{F}_n \in \mathbb{R}^{D \times D \times M}$.
Each filter $\T{F}_n$ in kernel $\T{K}$ performs convolution operation while sliding over all spatial locations on input feature maps.
Each filter produces one output feature map.


\subsection{Depthwise Separable Convolution}
{
Depthwise Separable Convolution (DSConv) (\cite{DSCONV}) consists of two sub-layers: depthwise convolution and pointwise convolution.
}
By decoupling tasks done by a standard convolution kernel, each of the two decomposed kernels independently performs its own task.
Note that standard convolution performs two tasks: 1) extracts spatial features (depthwise convolution) and 2) merges features along the channel dimension (pointwise convolution).
Decoupling the tasks allows DSConv to preserve the representation power while reducing the number of parameters and FLOPS at the same time.
Depthwise convolution (DWConv) kernel consists of several $D \times D$ 2-dimensional filters.
The number of 2-dimension filters is the same as that of input feature maps.
Each filter is applied on the corresponding input feature map, and produces an output feature map.
Pointwise convolution (PWConv), or $1\times 1$ convolution, is a standard convolution with kernel size 1. %
DSConv is defined as:
\begin{align}
\T{O}'_{h',w',m} &= \sum^{D}_{i=1} \sum^{D}_{j=1} \T{D}_{i,j,m} \cdot \T{I}_{h_i,w_j,m} \label{eq_prelim:mobileconv_dw} \\
\T{O}_{h',w',n} &= \sum^{M}_{m=1} \mathbf{P}_{m,n} \cdot \T{O}'_{h',w',m} \label{eq_prelim:mobileconv_pw}
\end{align}
where $\T{D}_{i,j,m}$ and $\mathbf{P}_{m,n}$ are depthwise convolution kernel and pointwise convolution kernel, respectively.
$\T{O}'_{h',w',m} \in \mathbb{R}^{H' \times W' \times M}$ denotes intermediate feature maps.
DSConv performs DWConv on input feature maps $\T{I}_{h_i,w_j,m}$ using \eqref{eq_prelim:mobileconv_dw},
and generates intermediate feature maps $\T{O}'_{h',w',m}$.
Then, DSConv performs PWConv on $\T{O}'_{h',w',m}$ using \eqref{eq_prelim:mobileconv_pw},
and generates output feature maps $\T{O}_{h',w',n}$.





\begin{table}[t!]
\caption{Symbols.}
\label{tab:symbol}
\centering
\resizebox{0.6\textwidth}{!}{
\begin{tabular}{c l}
\toprule
\textbf{Symbol} & \textbf{Description} \\
\midrule
$\T{K}$ & convolution kernel of size $\mathbb{R}^{D \times D \times M \times N}$ \\
$\T{I}$ & input feature maps of size $\mathbb{R}^{H \times W \times M}$ \\
$\T{O}$ & output feature maps of size $\mathbb{R}^{H' \times W' \times N}$ \\
$D$ & height and width of kernel (kernel size) \\
$M$ & number of input feature map (input channels) \\
$N$ & number of output feature map (output channels)\\
$H$ & height of input feature map \\
$W$ & width of input feature map \\
$H'$ & height of output feature map \\
$W'$ & width of output feature map \\
$\odot_{E}$ & Generalized Elementwise Product (GEP) \\
\bottomrule
\end{tabular}
}
 \end{table}

\section{Proposed Method}

In this section, we describe \method, our proposed method for compressing CNN using depthwise separable convolution.
We first present an overview of \method.
Then, we define Generalized Elementwise Product (GEP), a key mathematical formulation to generalize depthwise separable convolution.
We propose {\method} and explain why {\method} can replace standard convolution.
Then, we propose rank-k {\method}, which extends the basic \method.
We show that \method can be easily integrated into a branch architecture for compressing CNN.
We discuss relations of GEP and other modules based on depthwise separable convolution.
Finally, we theoretically analyze the performance of \method.

\subsection{Overview}
\label{subsec:overview}

\method compresses a trained CNN model which consists of standard convolution.
There are several challenges to be tackled to compress CNN using depthwise separable convolution.
\begin{enumerate}
	\item \textbf{Utilizing the knowledge of a trained model.}
	The knowledge of a trained model improves the performance of a compressed model with depthwise separable convolution.
	How can we utilize trained standard convolution kernels to better train depthwise separable convolution kernels?
	\item \textbf{Minimize accuracy gap.}
 	Existing modules based on depthwise separable convolution suffers from accuracy drop when we replace standard convolution with those modules.
	How can we design a depthwise separable convolution module to minimize the gap of accuracy between the trained model and the compressed model?
	\item \textbf{Flexibility.}
How can we provide a flexible method to adjust the accuracy and efficiency of a compressed model for various application scenarios?
	\end{enumerate}

To address the above challenges, we propose the following three main ideas which we elaborate in detail in the following subsections.

\begin{enumerate}
	\item \textbf{Formulating the precise relation} between standard convolution and depthwise separable convolution
allows us to fit depthwise separable convolution kernels to trained standard convolution kernels.
	\item \textbf{Carefully determining the order} of depthwise convolution and pointwise convolution improves the accuracy.
	\item  \textbf{Generalizing the formulation} provides an effective trade-off between accuracy and efficiency.
\end{enumerate}

Given a trained model, we first replace the standard convolution with \method which performs depthwise convolution after pointwise convolution.
We fit kernels of \method to those of standard convolution based on our precise formulation of the relation between standard convolution and depthwise separable convolution.
Then, we fine-tune the entire model using training data.

\subsection{Generalized Elementwise Product (GEP)}
\label{subsec:GEP}

The first challenge of compression with depthwise separable convolution is how to exploit the knowledge stored in a trained CNN model.
A naive approach would replace standard convolution with heuristic modules based on depthwise separable convolution and randomly initialize the kernels;
however, it does not leads to a good accuracy since it cannot use the knowledge of the trained model.
Our goal is to precisely figure out the relation between standard convolution and depthwise separable convolution, such that we fit the kernels of depthwise separable convolution to those of standard convolution.

For the purpose,
we define Generalized Elementwise Product (GEP), a generalized elementwise product for two operands of different shapes, to generalize the formulation of the relation between standard convolution and depthwise separable convolution.
%
Before generalizing the formulation, we give an example of formulating the relation between standard convolution and depthwise separable convolution.
Suppose we have a 4-order standard convolution kernel $\T{K} \in \mathbb{R}^{I \times J \times M \times N}$, a 3-order depthwise convolution kernel $\T{D} \in \mathbb{R}^{I \times J \times M}$, and a pointwise convolution kernel $\mathbf{P} \in \mathbb{R}^{M \times N}$.
Let $\T{K}_{i,j,m,n}$ be $(i,j,m,n)$-th element of $\T{K}$,
$\T{D}_{i,j,m}$ be $(i,j,m)$-th element of $\T{D}$, and $\mathbf{P}_{m,n}$ be $(m,n)$-th element of $\mathbf{P}$.
Then, it can be shown that applying depthwise convolution with $\T{D}$ and pointwise convolution with $\mathbf{P}$ is equivalent to applying standard convolution kernel $\T{K}$ where $\T{K}_{i,j,m,n} = \T{D}_{i,j,m} \cdot \mathbf{P}_{m,n}$ (see Theorem~\ref{theo:theo1}).



To formally express this relation, we define Generalized Elementwise Product (GEP) as follows.
\begin{definition}[Generalized Elementwise Product]
\label{def:GEP}
Given a $p$-order tensor $\T{D} \in \mathbb{R}^{I_1 \times \cdots \times I_{p-1} \times M}$ and a $q$-order tensor $\T{P} \in \mathbb{R}^{M \times J_1 \times \cdots \times J_{q-1}}$, the Generalized Elementwise Product $\T{D} \odot_{E} \T{P}$ of $\T{D}$ and $\T{P}$ is defined to be the $(p+q-1)$-order tensor $\T{K} \in$$ \mathbb{R}^{I_1 \times \ldots \times I_{p-1} \times M \times J_1 \times \ldots \times J_{q-1}}$ where the last axis of $\T{D}$ and the first axis of $\T{P}$ are the common axes such that for all elements of $\T{K}$,
\begin{align*}
\T{K}_{i_1,...,i_{p-1},m,j_1,...,j_{q-1}} = \T{D}_{i_1,...,i_{p-1},m} \cdot \T{P}_{m,j_1,...,j_{q-1}} \tag*{\QEDB}
\end{align*}
\end{definition}

Contrary to Hadamard Product which is defined only when the shapes of the two operands are the same, Generalized Elementwise Product (GEP) deals with tensors of different shapes.
Now, we define a special case of Generalized Elementwise Product (GEP) for a third-order tensor and a matrix. 
%
\begin{definition}[Generalized Elementwise Product for a third order tensor and a matrix]
\label{def:GEP_special}
Given a third-order tensor $\T{D} \in \mathbb{R}^{I \times J \times M}$ 
and a matrix $\mathbf{P} \in \mathbb{R}^{M \times N}$, 
the Generalized Elementwise Product $\T{D} \odot_{E} \mathbf{P}$ of $\T{D}$ and $\T{P}$ is defined to be the tensor $\T{K} \in$$ \mathbb{R}^{I \times J \times M \times N}$ where the third axis of the tensor $\T{D}$ and the first axis of the matrix $\mathbf{P}$ are the common axes such that for all elements of $\T{K}$,
\begin{align*}
\T{K}_{i,j,m,n} = \T{D}_{i,j,m} \cdot \mathbf{P}_{m,n} \tag*{\QEDB}
\end{align*}
\end{definition}

We will see that a standard convolution with kernel $\T{K}$ is essentially
equivalent to GEP of a depthwise convolution kernel $\T{D}$ and a pointwise convolution kernel $\T{P}$.
This allows to fit the kernels of depthwise separable convolution to those of standard convolution.
GEP is a core module of \method, and its variants rank-$k$ \method and \method-branch.
We also show that GEP is a core primitive that helps us understand other modules based on DSConv.

\subsection{\method: Lightweight and Accurate Convolution}
\label{subsec:FALCON}

The second and the most important challenge is how to minimize the accuracy gap between a trained model and a compressed model.
Existing modules based on depthwise separable convolution fail to address the challenge since they are heuristically designed and randomly initialized.

We design \method, a novel lightweight and accurate convolution that replaces standard convolution.
\method is an efficient method with fewer parameters and computations than what the standard convolution requires.
In addition, \method has a better accuracy than competitors while having similar compression and computation reduction rates.
Our main idea is to construct \method with the following guidelines:
\begin{enumerate}
	\item
	\method utilizes one depthwise convolution and one pointwise convolution to approximate the trained standard convolution.
	As we will see in Theorem~\ref{theo:theo2}, applying one pointwise convolution and one depthwise convolution is equivalent to applying a standard convolution whose kernel is defined with GEP of depthwise and pointwise kernels.
	\item
	We observe that a typical convolution has more output channels than input channels.
In such a setting, performing depthwise convolution after pointwise convolution
would allow the depthwise convolution to extract more features from richer feature space;
on the other hand, performing pointwise convolution after depthwise convolution only combines features extracted from a limited feature space.
\end{enumerate}
Based on the guidelines, \method first applies pointwise convolution to generate an intermediate tensor $\T{O}' \in \mathbb{R}^{H \times W \times N}$ 
and then applies depthwise convolution.
%
%

\begin{figure} [t]
\centering
\includegraphics[width=0.5\textwidth]{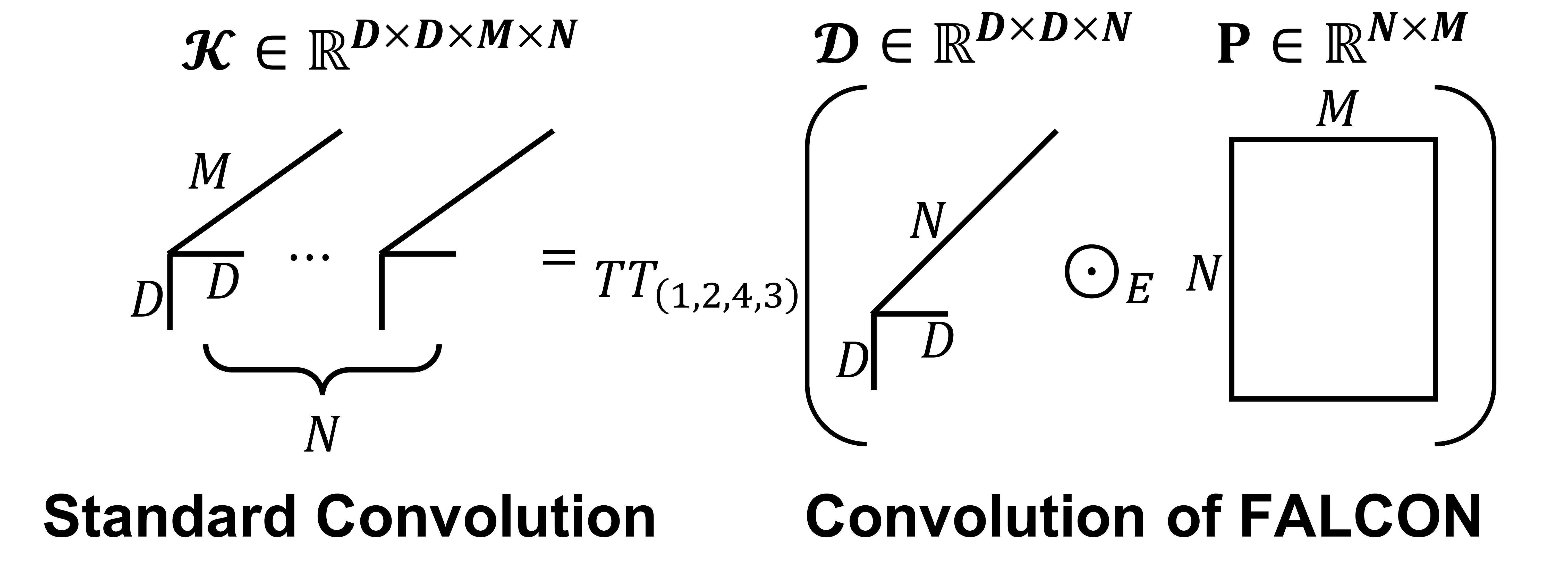}
\caption{
Relation between standard convolution and \method expressed with GEP.
The common axes correspond to the output channel-axis of standard convolution.
{$TT_{(1,2,4,3)}$ indicates tensor transpose operation to permute the third and the fourth dimensions of a tensor.}
}
\label{fig:GEP_falcon}
\end{figure}

We represent the relationship between standard convolution kernel $\T{K} \in \mathbb{R}^{D \times D \times M \times N}$ and \method using an GEP operation of
pointwise convolution kernel $\mathbf{P} \in \mathbb{R}^{N \times M}$ and
depthwise convolution kernel $\T{D} \in \mathbb{R}^{D \times D \times N}$ in Figure~\ref{fig:GEP_falcon}.
In \method, the kernel $\T{K}$ is represented as follows:
\begin{align}
\label{eq:gep_falcon}
\T{K} = TT_{(1,2,4,3)} (\T{D} \odot_{E} \mathbf{P}) \hspace{3mm}\text{s.t.} \hspace{3mm}
\T{K}_{i,j,m,n} = \mathbf{P}_{n,m} \cdot \T{D}_{i,j,n}.
\end{align}
where $TT_{(1,2,4,3)}$ indicates tensor transpose operation to permute the third and the fourth dimensions of a tensor.
Note that the common axis is the output channel axis of the standard convolution, unlike GEP for depthwise separable convolution where the common axis is the input channel axis of the standard convolution.

We show that applying \method is equivalent to applying standard convolution with a specially constructed kernel.
\begin{theorem}
\label{theo:theo2}
\method which applies
pointwise convolution with kernel $\mathbf{P} \in \mathbb{R}^{N \times M}$ and then
depthwise convolution with kernel $\T{D} \in \mathbb{R}^{D \times D \times N}$
is equivalent to applying standard convolution with kernel $\T{K} = TT_{(1,2,4,3)} (\T{D} \odot_{E} \mathbf{P})$. \QEDB
\end{theorem}

\begin{proof}
\label{proof:theo2}
Based on the Equation~\eqref{eq:gep_falcon}, we re-express
Equation~\eqref{eq:standardconv} by replacing the kernel $\T{K}_{i,j,m,n}$ with the pointwise convolution kernel $\mathbf{P}_{n,m}$ and the depthwise convolution kernel $\T{D}_{i,j,n}$. 
\begin{align*}
\T{O}_{h',w',n} =  \sum^{M}_{m=1} \sum^{D}_{i=1} \sum^{D}_{j=1} \mathbf{P}_{n,m} \cdot \T{D}_{i,j,n} \cdot \T{I}_{h_i,w_j,m}
\end{align*}
where $\T{I}_{h_i, w_j, m}$ is the $(h_i, w_j, m)$-th entry of the input $\T{I}$.
We split the above equation into the following two equations.
\begin{align}
\T{O}'_{h_i,w_j,n} &= \sum^{M}_{m=1} \mathbf{P}_{n,m} \cdot \T{I}_{h_i,w_j,m} \label{eq:falcon_pw1} \\
\T{O}_{h',w',n} &= \sum^{D}_{i=1} \sum^{D}_{j=1} \T{D}_{i,j,n} \cdot \T{O}'_{h_i,w_j,n} \label{eq:falcon_dw1}
\end{align}
where $\T{I}$, $\T{O}'$, and $\T{O}$ are the input, the intermediate, and the output {tensors} of convolution layer, respectively.
Note that \eqref{eq:falcon_pw1} and~\eqref{eq:falcon_dw1}
correspond to pointwise convolution and depthwise convolution, respectively.
	Therefore, the output {$\T{O}_{h',w',n}$} is equal to the output from applying \method.
\end{proof}


\begin{algorithm} [t]
	\SetNoFillComment
	\caption{\method}\label{alg:method}
	\begin{algorithmic} [1]
		\small
		\algsetup{linenosize=\small}
		\renewcommand{\algorithmicrequire}{\textbf{Input:}}
		\renewcommand{\algorithmicensure}{\textbf{Output:}}
		    \REQUIRE training data, and a trained model consisting of standard convolution
		    \ENSURE a trained compressed model

		\STATE \textbf{Stage 1: Replacement and initialization.} Replace each standard convolution with \method.
		Then, initialize the pointwise convolution kernel $\T{D}$ and the depthwise convolution kernel $\mathbf{P}$ of \method by fitting them to the standard convolution kernels of the trained model; i.e.,
$\T{D}, \mathbf{P} = \argmin_{\T{D'}, \mathbf{P'}} ||\T{K}- TT_{(1,2,4,3)} (\T{D'} \odot_{E} \mathbf{P'})||_F$. \label{alg:line:stage1}
	We repeatedly update $\T{D}$ and $\mathbf{P}$ by minimizing $||\T{K}- TT_{(1,2,4,3)} (\T{D'} \odot_{E} \mathbf{P'})||_F$ until a maximum number of iterations is exceeded.
		\STATE \textbf{Stage 2: Fine-tuning.} Train all the parameters from the initialized model using training data. \label{alg:line:stage2}
\\
	\end{algorithmic}
\end{algorithm}

Algorithm~\ref{alg:method} summarizes how \method works.
Given a trained model with standard convolution, we replace each standard convolution with \method.
Based on the equivalence, we fit the pointwise convolution and the depthwise convolution kernels $\T{D}$ and $\mathbf{P}$ of \method to the convolution kernels of the trained standard model; i.e.,
$\T{D}, \mathbf{P} = \argmin_{\T{D'}, \mathbf{P'}} ||\T{K}- TT_{(1,2,4,3)} (\T{D'} \odot_{E} \mathbf{P'})||_F$ (stage~\ref{alg:line:stage1} in Algorithm~\ref{alg:method}).
We repeatedly update $\T{D}$ and $\mathbf{P}$ by minimizing $||\T{K}- TT_{(1,2,4,3)} (\T{D'} \odot_{E} \mathbf{P'})||_F$ until a maximum number of iterations is exceeded.
$\T{D}$ and $\mathbf{P}$, obtained by minimizing $||\T{K}- TT_{(1,2,4,3)} (\T{D'} \odot_{E} \mathbf{P'})||_F$, 
allow a compressed model to learn from better initialization than random initialization.
After pointwise convolution and depthwise convolution, we add batch-normalization and ReLU activation function.
Finally, we fine-tune all the parameters of the comrpessed model with \method using training data (stage~\ref{alg:line:stage2} in Algorithm~\ref{alg:method}).
%
We note that \method significantly reduces the numbers of parameters and FLOPs compared to standard convolution.

\begin{figure*} [t]
	\centering
	\subfloat[StConv-branch] {\includegraphics[width=0.25\textwidth,height=0.215\textheight]{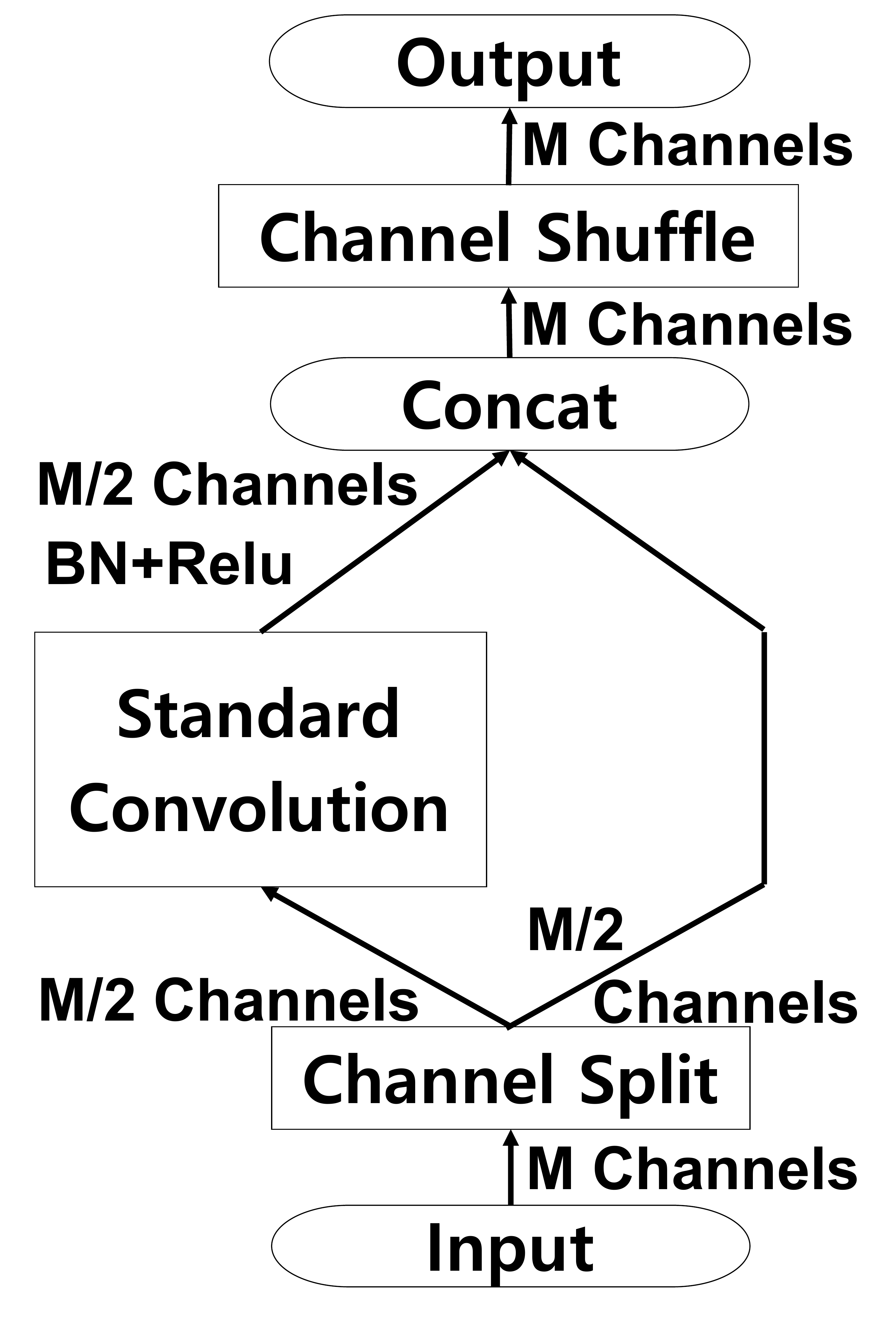}\label{fig:stconv_branch}}
	\subfloat[\method-branch (proposed)] {\includegraphics[width=0.3\textwidth,height=0.215\textheight]{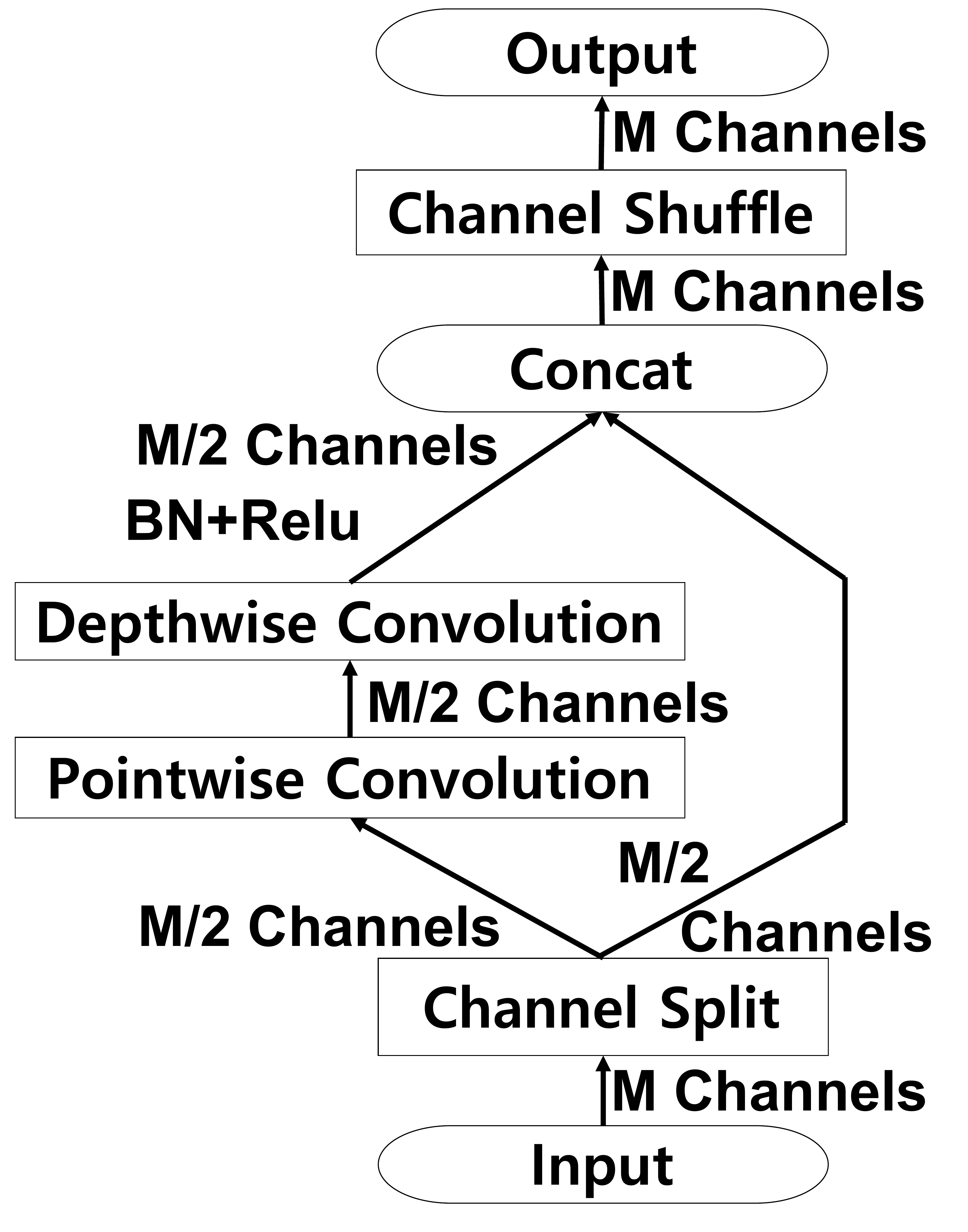}\label{fig:falcon_branch}}
	\caption{
	Comparison of architectures. BN denotes batch-normalization. Relu and Relu6 are activation functions.
	\label{fig:conv_archi}
(a) Standard convolution with branch (StConv-branch).
(b) \method-branch which combines \method with StConv-branch.
}
\end{figure*}



\subsection{Rank-$k$ \method}
\label{subsec:rank_falcon}

The third challenge is to flexibly adjust the accuracy and efficiency of a compressed model for various application scenarios.
In general, a compressed model provides a lower accuracy while requiring smaller number of parameters compared to the original model.
How can we make a compressed model improve its accuracy while sacrificing compression?

We propose rank-$k$ \method,
an extended version of \method that further improves accuracy while sacrificing a bit of compression and computation reduction rates.
The main idea is to perform $k$ independent \method operations
and sum up the result.
Then, we apply batch-normalization (BN) and ReLU activation function to the summed result.
Since each \method operation requires independent parameters for pointwise convolution and depthwise convolution, the number of parameters increases and thus the compression and the computation reduction rates decrease; however, it improves accuracy by enlarging the model capacity.
We formally define the rank-$k$ \method with GEP as follows.



\begin{definition}[Rank-$k$ \method with Generalized Elementwise Product]
Rank-$k$ \method expresses
standard convolution kernel $\T{K} \in \mathbb{R}^{D \times D \times M \times N}$
as GEP of depthwise convolution kernel $\T{D}^{(r)} \in \mathbb{R}^{D \times D \times N}$ and pointwise convolution kernel $\mathbf{P}^{(r)} \in \mathbb{R}^{N \times M}$ for $r = 1,2, ..., k$.
\begin{align*}
\T{K} = \sum^{k}_{r=1} TT_{(1,2,4,3)} (\T{D}^{(r)} \odot_{E} \mathbf{P}^{(r)})
\quad \text{s.t.} \quad 
\T{K}_{i,j,m,n} = \sum^{k}_{r=1} \mathbf{P}^{(r)}_{n,m} \cdot \T{D}^{(r)}_{i,j,n}\tag*{\QEDB}
\end{align*}
\end{definition}
For each $r = 1,2,...,k$, we construct the tensor $\T{K}^{(r)}$ using GEP of
the depthwise convolution kernel $\T{D}^{(r)}$ and
the pointwise convolution kernel $\mathbf{P}^{(r)}$.
Then, we construct the standard kernel $\T{K}$ by the sum of the tensors $\T{K}^{(r)}$ for all $r$.



\subsection{\method-Branch}
\label{subsec:method:branch}

\method can be easily integrated into a CNN architecture called standard convolution with branch (StConv-branch), which consists of two branches: standard convolution on the left branch and a residual connection on the right branch.
ShuffleNetV2 (\cite{ShuffleNetV2}) improved the performance of CNN by applying depthwise and pointwise convolutions on the left branch of StConv-branch.
Since \method replaces standard convolution,
we observe that StConv-branch can be easily compressed by applying \method on the left branch.

Figure~\ref{fig:stconv_branch} illustrates StConv-branch. 
StConv-branch first splits an input in half along the depth dimension.
A standard convolution operation is applied to one half, and no operation to the other half.
The two are concatenated along the depth dimension, and an output is produced by shuffling the channels of the concatenated tensor.
As shown in Figure~\ref{fig:falcon_branch}, \method-branch is constructed by replacing the standard convolution branch (left branch) of StConv-branch with \method.
Advantages of \method-branch are that
1) the branch architecture improves the efficiency since convolutions are applied to only half of input feature maps and that
2) \method further compresses the left branch effectively.
\method-branch is initialized by fitting kernels of \method to the standard convolution kernel of the left branch of StConv-branch.

\subsection{Relation to Modules using DSConv}
\label{subsec:DSConvGEP}

Recent works (\cite{MobileNet,MobileNetV2,MobileNetV3,ShuffleNet,ShuffleNetV2}) have constructed lightweight architectures by designing modules based on depthwise separable convolution.
Our proposed formulation GEP is crucial to understand modules based on depthwise separable convolution.
We discuss the relation between GEP and those modules.

\subsubsection{Depthwise-Pointwise Convolution Module (DPConv)}
%
DPConv, which is used in MobilenetV1 (\cite{MobileNet}), performs pointwise convolution after depthwise convolution.
%
%
We show that applying depthwise separable convolution with $\T{D}$ and $\mathbf{P}$ is equivalent to applying standard convolution with a kernel $\T{K}$ which is constructed from $\T{D}$ and $\mathbf{P}$ using GEP.
\begin{theorem}
\label{theo:theo1}
Applying depthwise separable convolution with depthwise convolution kernel $\T{D} \in \mathbb{R}^{D \times D \times M}$ and pointwise convolution kernel $\mathbf{P} \in \mathbb{R}^{M \times N}$ is equivalent to applying standard convolution with kernel $\T{K} = \T{D} \odot_{E} \mathbf{P}$. \qed
\end{theorem}

\begin{proof}
\label{proof:theo1}
From the definition of \gep,
$\T{K}_{i,j,m,n} = \T{D}_{i,j,m} \cdot \mathbf{P}_{m,n}$.
We replace the kernel $\T{K}_{i,j,m,n}$ with the depthwise convolution kernel $\T{D}_{i,j,m}$ and the pointwise convolution kernel $\mathbf{P}_{m,n}$. 
\begin{align*}
\T{O}_{h',w',n} = \sum^{D}_{i=1} \sum^{D}_{j=1} \sum^{M}_{m=1} \T{D}_{i,j,m} \cdot \mathbf{P}_{m,n} \cdot \T{I}_{h_i,w_j,m}
\end{align*}
where $\T{I}_{h_i, w_j, m}$ is the $(h_i, w_j, m)$-th entry of the input.
We split the above equation into the following two equations.
\begin{align}
\T{O}'_{h',w',m} &= \sum^{D}_{i=1} \sum^{D}_{j=1} \T{D}_{i,j,m} \cdot \T{I}_{h_i,w_j,m} \label{mobileconv_dw}\\
\T{O}_{h',w',n} &= \sum^{M}_{m=1} \mathbf{P}_{m,n} \cdot \T{O}'_{h',w',m} \label{mobileconv_pw}
\end{align}
where $\T{O}'_{h',w',m} \in \mathbb{R}^{H' \times W' \times M}$ is an intermediate tensor. 
Note that \eqref{mobileconv_dw} and~\eqref{mobileconv_pw}
correspond to applying a depthwise convolution and a pointwise convolution, respectively.
Therefore, the output $\T{O}'_{h',w',m}$ is equal to the output after applying depthwise separable convolution, DPConv, used in MobilenetV1.
\end{proof}

\subsubsection{Pointwise-Depthwise-Pointwise Convolution Module (PDPConv)}
PDPConv, which is used in \cite{MobileNetV2,MobileNetV3}, can be understood as DPConv preceded with an additional pointwise convolution.
This module is equivalent to a layer consisting of
pointwise convolution followed by standard convolution expressed by GEP.

%

\subsubsection{Group-Depthwise-Group Convolution Module (GDGConv)}
GDGConv, which is used in Shufflenet (\cite{ShuffleNet}), consists of four sublayers, group pointwise convolution, channel shuffle, depthwise convolution, and another group pointwise convolution, as well as a shortcut.
We examine the relation between standard convolution and the last two convolutions of GDGConv using GEP as follows.
Let $g$ be the number of groups and $\T{K}^l \in \mathbb{R}^{D \times D \times \frac{M}{g} \times \frac{N}{g}}$ be the $l$th group standard convolution kernel.
Note that the input tensor is split into $g$ group tensors along the channel axis and each group standard convolution performs a convolution operation to its corresponding group tensor.
Then, the relation of $l$-th group standard convolution kernel $\T{K}^l \in \mathbb{R}^{D \times D \times \frac{M}{g} \times \frac{N}{g}}$ with regard to {$l$-th depthwise convolution kernel $\T{D}^l \in \mathbb{R}^{D \times D \times \frac{M}{g}}$} and $l$-th pointwise group convolution kernel $\mathbf{P}^l \in \mathbb{R}^{\frac{M}{g} \times \frac{N}{g}}$ is
\begin{align*}
\T{K}^l = \T{D}^l \odot_{E} \mathbf{P}^l \hspace{5mm}\text{s.t.} \hspace{5mm}
\T{K}^l_{i,j,m_g,n_g} = \T{D}^l_{i,j,m_g} \cdot \mathbf{P}^l_{m_g,n_g}
\end{align*}
where $m_g = 1,2,...,\frac{M}{g}$ and $n_g = 1,2,...,\frac{N}{g}$.
Each group standard convolution is equivalent to the combination of a depthwise convolution and a pointwise convolution, and thus
easily expressed with GEP. 
%

\subsubsection{PDPConv with split (PDPConv-split)}
PDPConv-split module is similar to PDPConv except that
1) the input channels are split into two branches,
2) the left branch undergoes the same convolutions as those of PDPConv which are expressed as GEP, and
3) the right branch is an identity connection.
ShufflenetV2 (\cite{ShuffleNetV2}) is constructed by adopting this module.

Since we can easily replace standard convolution with the modules mentioned above,
we compare \method with such modules in the experiment section.
Based on Theorem~\ref{theo:theo1}, we initialize kernels of DPConv similar to kernels of \method. Kernels $\T{D}$ and $\mathbf{P}$ of DPConv are fitted from the pretrained standard convolution kernel $\T{K}$; i.e., $\T{D}, \mathbf{P} = \argmin_{\T{D'}, \mathbf{P'}} ||\T{K}- \T{D'} \odot_{E} \mathbf{P'}||_F$.
In contrast to DPConv,
compressed models with PDPConv, GDGConv, and PDPConv-split fail to use GEP to initialize kernels and are trained from scratch, 
since they are not equivalent to standard convolution.
Therefore, the models with the three convolution modules are trained from scratch.

\subsection{Analysis}
\label{subsec:metrics}
We evaluate the compression and the computation reduction of \method and rank-$k$ \method.
All the analysis is based on one convolution layer.
The comparison of the numbers of parameters and FLOPs of \method and other competitors is
in Table~\ref{table:param_flops}. 

\begin{table}[!t]
	\centering
		{
	\caption{Numbers of parameters and FLOPs of \method and competitors. Symbols are described in Table~\ref{tab:symbol}. 
}
	\label{table:param_flops}
	\resizebox{0.8\textwidth}{!}{
	\begin{tabular}{c|cc}
		\toprule
		\textbf{Convolution} & \textbf{$\#$ of parameters} & \textbf{$\#$ of FLOPs}  \\
		\midrule
		\method   & $MN + D^2N$ & $HWMN + H'W'D^2N$ \\
		\method-branch & $\frac{1}{4}M^2 + \frac{1}{2}D^2M$ & $\frac{1}{4}HWM^2 + \frac{1}{2}HWD^2M$ \\
		DPConv & $MN + D^2M$ & $HWD^2M + H'W'MN$ \\
		PDPConv & $tM^2+tD^2M+tMN$ & $tHWM^2+tH'W'D^2M+tH'W'MN$ \\
		GDGConv & $\frac{1}{4}(\frac{MN}{g}+D^2N+\frac{N^2}{g})$ & $\frac{1}{4}(\frac{HWMN}{g}+H'W'D^2N+\frac{H'W'N^2}{g})$ \\
		PDPConv-split & $\frac{1}{2}(M^2 + D^2M)$ & $\frac{1}{2}HW(M^2 + D^2M)$ \\
		StConv-branch & $\frac{1}{4}D^2M^2$ &  $\frac{1}{4}HWD^2M^2$ \\
		Standard convolution & $D^2MN$ & $H'W'D^2MN$ \\
 		\bottomrule
	\end{tabular}}}
\end{table}

\subsubsection{\method}
We analyze the compression and the computation reduction rates of \method in Theorems~\ref{theo:cr} and~\ref{theo:crr}.
\begin{theorem}
\label{theo:cr}
Compression Rate ($CR$) of \method is given by
{
\begin{align*}
CR = \displaystyle{\frac{\text{\# of parameters in standard convolution}}{\text{\# of parameters in \method}}} = \frac{D ^2MN} {MN + D^2N}
\end{align*}}
where $D^2$ is the size of standard kernel, $M$ is the number of input channels, and $N$ is the number of output channels. \QEDB
\end{theorem}
\begin{proof}
{
Standard convolution kernel has $D^2MN$ parameters.
\method includes pointwise convolution and depthwise convolution which requires $MN$ and $D^2N$ parameters, respectively.
Thus, the compression rate of \method is $CR = \frac{D^2MN}{MN + D^2N}$.
}
\end{proof}

\begin{theorem}
\label{theo:crr}
Computation Reduction Rate ($CRR$) of \method is described as:
{
\begin{align*}
CRR &= \displaystyle{\frac{\text{\# of FLOPs in standard convolution}}{\text{\# of FLOPs in \method}}}
= \frac{H'W'MD^2N}{HWMN + H'W'D^2N}
\end{align*}
}
where $H'$ and $W'$ are the height and the width of output, respectively, and $H$ and $W$ are the height and the width of input, respectively.
\end{theorem}

\begin{proof}
The standard convolution operation requires $H'W'D^2MN$ FLOPs (\cite{NVIDIA}).
\method includes pointwise convolution and depthwise convolution.
Pointwise convolution has kernel size $D=1$ with stride $s=1$ and no padding, so the intermediate tensor $\T{O}'$ has the same height and width as those of the input feature maps.
Thus, pointwise convolution needs $HWMN$ FLOPs.
Depthwise convolution has the number of input channel $M=1$, so it needs $H'W'D^2N$ FLOPs.
The total FLOPs of \method is $HWMN + H'W'D^2N$, thus the computation reduction rate of \method is $\displaystyle{CRR = \frac{H'W'D^2MN}{HWMN + H'W'D^2N}}$.
\end{proof}

\subsubsection{Rank-k \method}
We analyze the compression and computation reduction rates of rank-$k$ \method in Theorem~\ref{theo:rankK}.
\begin{theorem}
\vspace{-2mm}
\label{theo:rankK}
Compression Rate ($CR_k$) and Computation Reduction Rate ($CRR_k$) of rank-k \method are described as:
\begin{align*}
CR_k = \frac{CR}{k} \hspace{10mm} CRR_k = \frac{CRR}{k} \tag*{\QEDB}
\end{align*}
\end{theorem}

\begin{proof}
The numbers of parameters and FLOPs increase by k times since rank-k \method duplicates \method for k times.
Thus, the compression rate and the computation reduction rate are calculated as $CR_k = \displaystyle{\frac{CR}{k}}$ and $CRR_k = \displaystyle{\frac{CRR}{k}}$.
\end{proof}

\section{Experiments}

We validate the performance of \method through extensive experiments.
We aim to answer the following questions:

\begin{itemize*}
{
\item \textbf{Q1. Performance.}
Which method gives the best accuracy for a given compression and computation reduction rate?

\item \textbf{Q2. Ablation Study.}
How do the initialization and the alignment of depthwise separable convolution affect the performance of \method?
}

%

\item \textbf{Q3. Rank-$k$ \method.}
How do the accuracy, the number of parameters, and the number of FLOPs change
as the rank $k$ increases in \method?

\item \textbf{Q4. Comparison with Tensor decomposition.}
How effectively does \method compress CNN models compared to CP and Tucker decomposition methods?

\hide{
\item \textbf{Q5. Effectiveness of Initialization.}
How does the initialization affect the fine-tuning step?}

\end{itemize*}

\begin{table}[t]
\caption{Datasets.}
\centering
\label{tab:data}
\resizebox{0.8\textwidth}{!}{
\begin{tabular}{lrrrr}
\toprule
\textbf{dataset} & \textbf{\# of classes} & \textbf{input size} & \textbf{\# of train} & \textbf{\# of test} \\
\midrule
CIFAR-10\tablefootnote{\url{https://www.cs.toronto.edu/~kriz/cifar.html}\label{r}}  & 10 & $32 \times 32 \times 3$ & $10 \times 6000$ & $10000$ \\
CIFAR-100$^{\ref{r}}$ & 100 & $32 \times 32 \times 3$ & $100 \times 600$ & $10000$ \\
SVHN\tablefootnote{\url{http://ufldl.stanford.edu/housenumbers/}} & 10 & $32 \times 32$ & $73257$ & $26032$ \\
ImageNet\tablefootnote{\url{http://www.image-net.org}} & 1000 & $224 \times 224 \times 3$ & $1.2 \times 10^6$ & $150000$ \\
\bottomrule
\end{tabular}
}
\end{table}

\subsection{Experimental Setup}
\label{subsec: exp_setup}


We construct all models using Pytorch framework.
All the models are trained and tested on a machine with GeForce GTX 1080 Ti GPU.
We perform image classification task on four famous datasets in Table~\ref{tab:data}: CIFAR10, CIFAR100, SVHN, and ImageNet.

\subsubsection{Prediction Models}
For CIFAR10, CIFAR100, and SVHN datasets, we choose VGG19 and ResNet34 to evaluate the performance.
We shrink the sizes of both models since the sizes of these three datasets are smaller than that of Imagenet.
On both models, we replace all standard convolution layers (except for the first convolution layer) with those of \method or other competitors in order to compress and accelerate the model.
For ImageNet, we choose VGG16\_BN (VGG16 with batch normalization after every convolution layer) and ResNet18.
We use the pretrained model from Pytorch model zoo as the baseline model with standard convolution. 

In VGG19, we reduce the number of fully connected layers and the number of features in fully connected layers: three large fully connected layers (4096-4096-1000) in VGG19 are replaced with two small fully connected layers (512-10 or 512-100).
In ResNet34, we remove the first $7 \times 7$ convolution layer and max-pooling layer
since the input size ($32 \times 32$) of these datasets is smaller than the input size ($224 \times 224$) of ImageNet.

\subsubsection{Competitors and Evaluation}
\label{sec:exp:setup:competitor}

We compare \method and \method-branch with four convolution units consisting of depthwise convolution and pointwise convolution: DPConv, PDPConv, GDGConv, PDPConv-split.
%
To evaluate the effectiveness of fitting depthwise and pointwise convolution kernels to standard convolution kernel,
we build GEP-in which is DPConv where kernels $\T{D}$ and $\mathbf{P}$ are fitted from the pretrained standard convolution kernel $\T{K}$; i.e., $\T{D}, \mathbf{P} = \argmin_{\T{D'}, \mathbf{P'}} ||\T{K}- \T{D'} \odot_{E} \mathbf{P'}||_F$.
We also compare \method with two tensor decomposition methods, CP and Tucker.
We take each standard convolution layer (StConv) as a unit, and replace StConv with those from \method or other competitors.
We evaluate the classification accuracy, the number of parameters in the model, and the number of FLOPs needed for forwarding one image.

\subsubsection{Optimization}

For kernel fitting of \method, we use the AdamW optimizer.
We set the learning rate to $0.05$ for CIFAR10, and $0.001$ for the remaining datasets.
We choose the maximum number of iterations between $600$ to $1,000$.
%
For training and fine-tuning of all methods, we use the SGD optimizer.
The best learning rate is chosen among $0.5$, $0.1$, and $0.01$.
Weight decay is set to $0.0001$ and momentum is set to $0.9$. For CIFAR10, CIFAR100, and SVHN, the models are trained and fine-tuned for $350$ epochs. At epochs $150$ and $250$, the learning rate is decreased by a factor of $10$. For ImageNet, the models are trained and fine-tuned for $90$ epochs;  the learning rate is decreased by a factor of $10$ at epochs $30$ and $60$.

\subsubsection{Compression Methods}

The hyperparameter settings for each compression method is as follows.

\textbf{\method.}
When replacing StConv with \method, we use the same setting as that of StConv.
That is, if there are BN and ReLU after StConv, we add BN and ReLU at the end of \method; if there is only ReLU after StConv, we add only ReLU at the end of \method.
This is because \method is initialized by approximating the StConv kernel using GEP;
using the same setting for BN and ReLU as StConv is more effective for \method to approximate the StConv.
%
%
%
%
We fit the pointwise convolution kernel and the depthwise convolution kernel of \method to the pretrained standard convolution kernel using GEP,
similarly to GEP-in.
%
%
Rank-$k$ \method uses the same fitting method.

\textbf{DPConv.}
DPConv has the most similar architecture as \method among competitors, and thus DPConv has nearly the same number of parameters as that of \method.
As in \method, the existence of BN and ReLU at the end of DPConv depends on that of StConv.

\textbf{PDPConv.}
In PDPConv,
we adjust the numbers of parameters and FLOPs by changing the expansion ratio $t$, which is represented as `PDPConv-$t$'.
We choose $t=0.5$ as the baseline PDPConv to compare with \method, since two pointwise convolutions impose lots of parameters and FLOPs to PDPConv.

\textbf{GDGConv.}
In GDGConv,
we adjust the numbers of parameters and FLOPs by changing the width multiplier $\alpha$ (\cite{MobileNet}) and the number $g$ of groups, which is represented as `GDGConv $\alpha \times$(g=$g$)'.
Note that the width multiplier is used to adjust the number $M$ of input channels and the number $N$ of output channels of a convolution layer; if the width multiplier is $\alpha$, the numbers of input and output channels become $\alpha M$ and $\alpha N$, respectively.
We also observe that GDGConv does not cooperate well with ResNet: ResNet34 with GDGConv does not converge.
We suspect that residual block and GDGConv may conflict with each other because of redundant residual connections:
a gradient may not find the right path towards previous layers.
For this reason, we delete the shortcuts of all residual blocks in ResNet34 when using GDGConv.

\textbf{PDPConv-split.}
In PDPConv-split, we also adjust the number of parameters and FLOPs by changing the width multiplier $\alpha$, which is represented as 'PDPConv-split $\alpha \times$'.
Other operations of PDPConv-split stay the same as in \cite{ShuffleNetV2}.

\textbf{CP and Tucker.}
The ranks for CP and Tucker decomposition methods are determined from Variational Bayesian Matrix Factorization (\cite{VBMF}).

\begin{figure*} [t]
	\centering
	\subfloat{\includegraphics[width=0.7\textwidth]{FIG/pythoncode/LEGEND1.pdf}}\vspace{-4mm} \\
	\setcounter{subfigure}{0}
	\subfloat[VGG19-CIFAR10] {\includegraphics[width=0.27\textwidth]{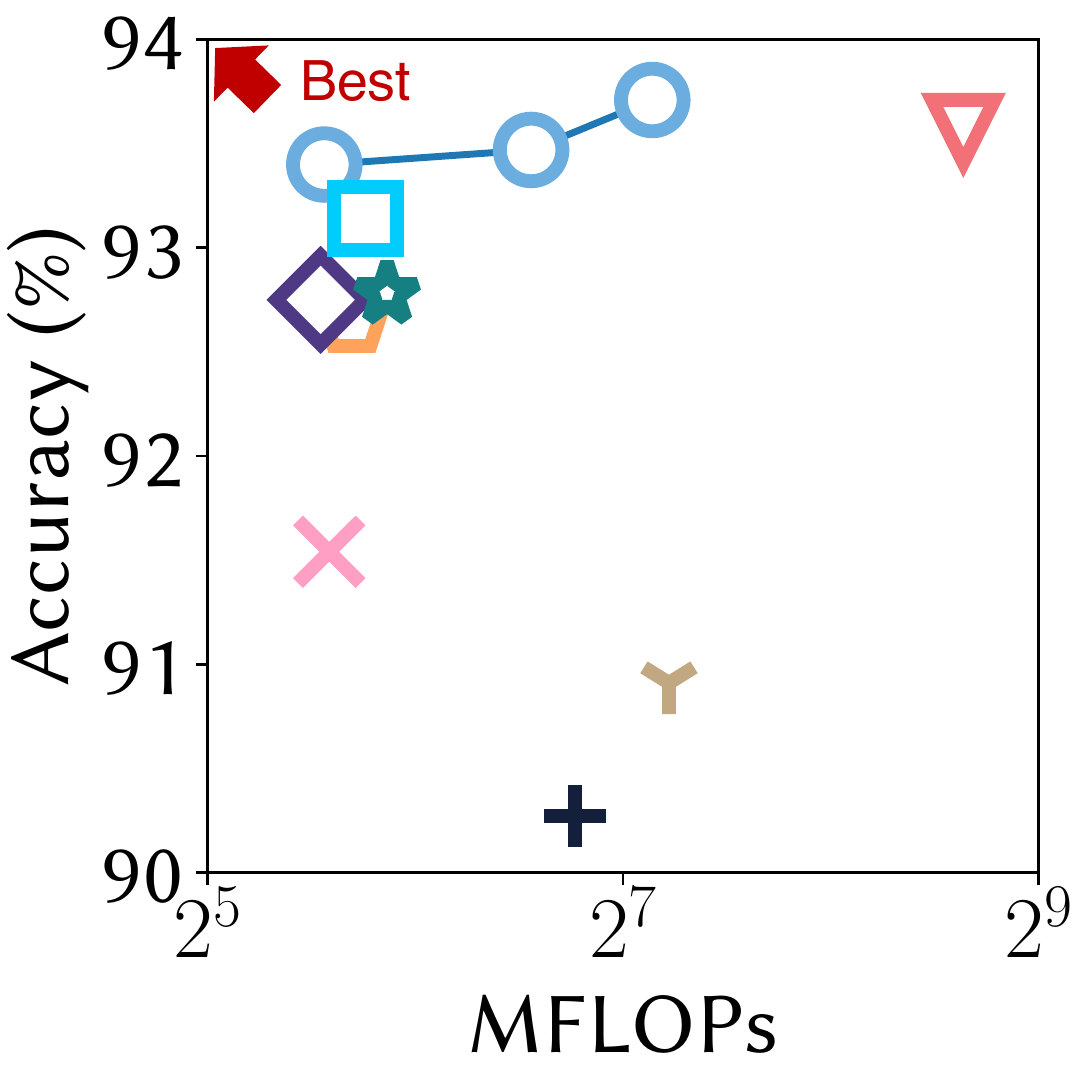}\label{fig:vgg19_cifar10_flops}}
	\subfloat[VGG19-SVHN] {\includegraphics[width=0.275\textwidth]{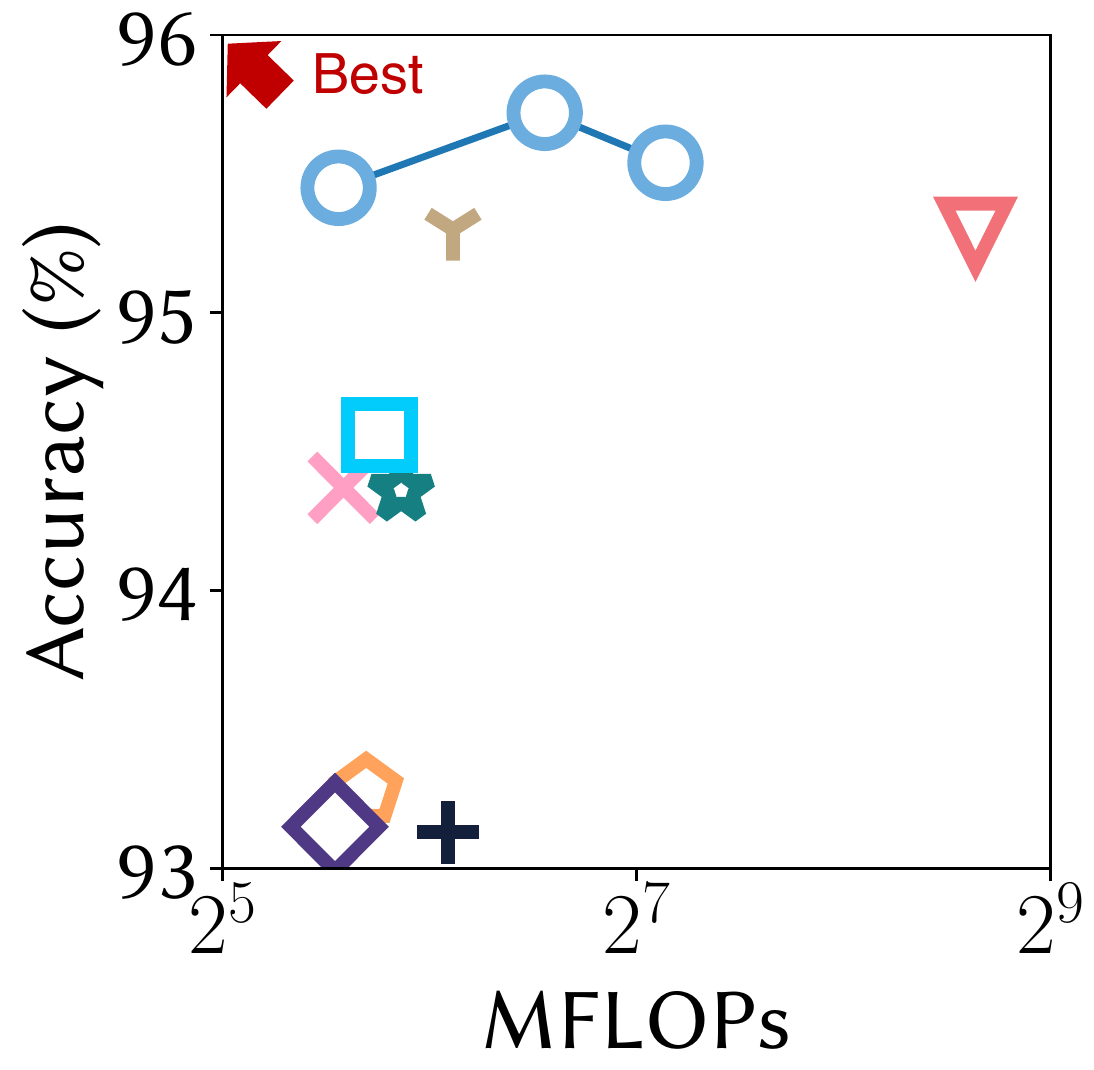}\label{fig:vgg19_svhn_flops}}
	\subfloat[VGG16-ImageNet] {\includegraphics[width=0.27\textwidth]{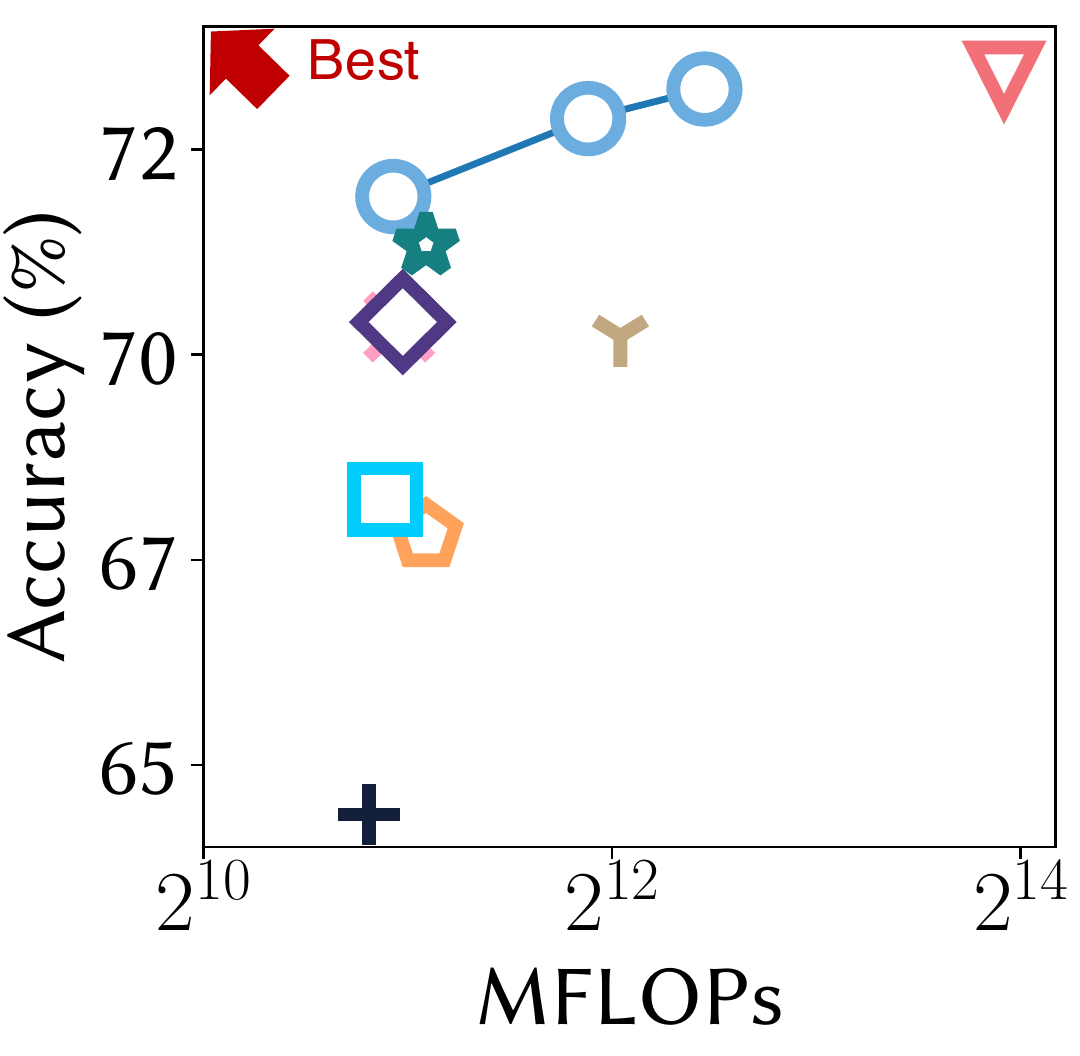}\label{fig:vgg19_imagenet_flops}} \\ \vspace{-2mm}
	\subfloat[ResNet34-CIFAR10] {\includegraphics[width=0.27\textwidth]{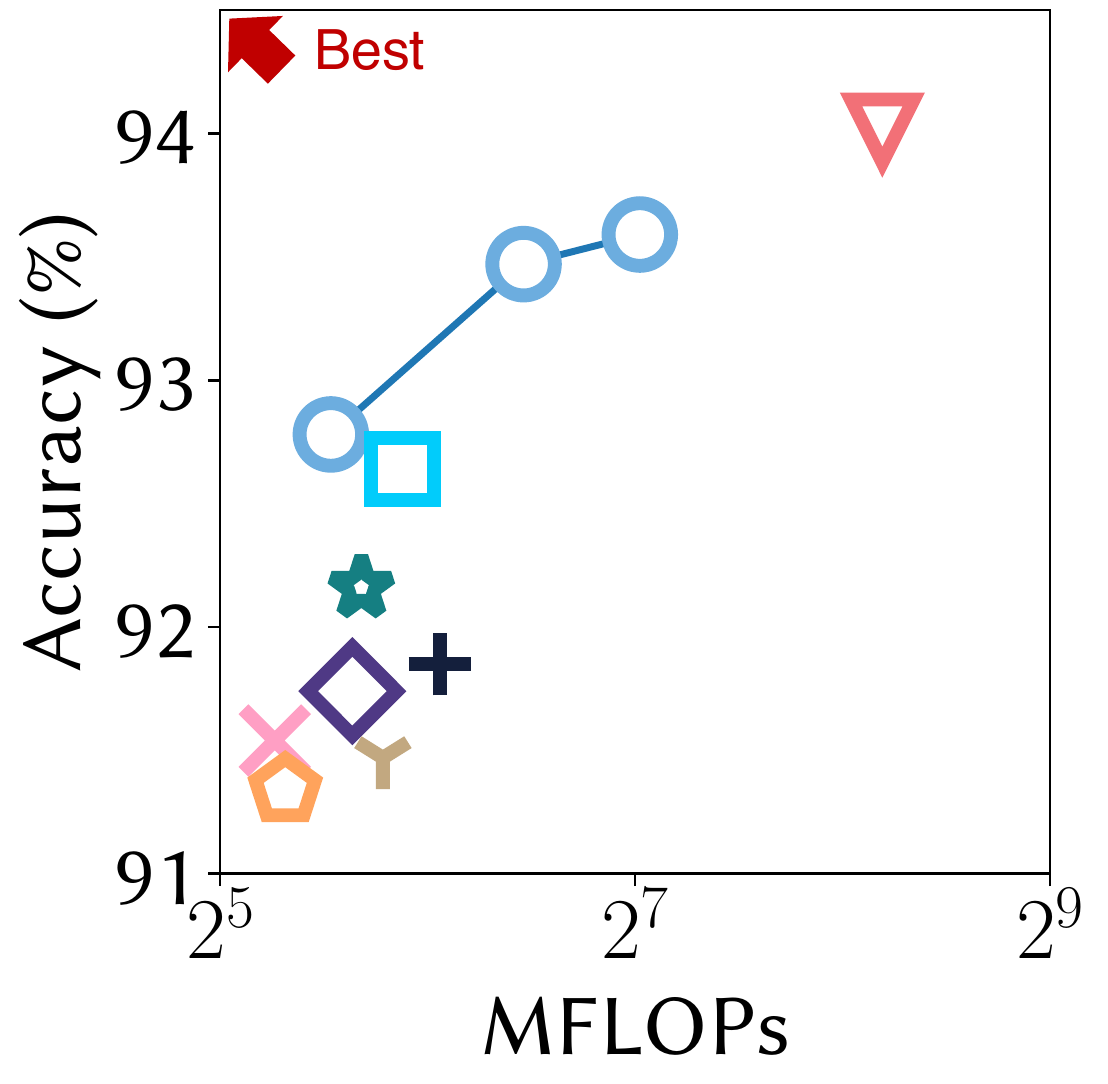}\label{fig:resnet_cifar10_flops}}
	\subfloat[ResNet34-SVHN] {\includegraphics[width=0.275\textwidth]{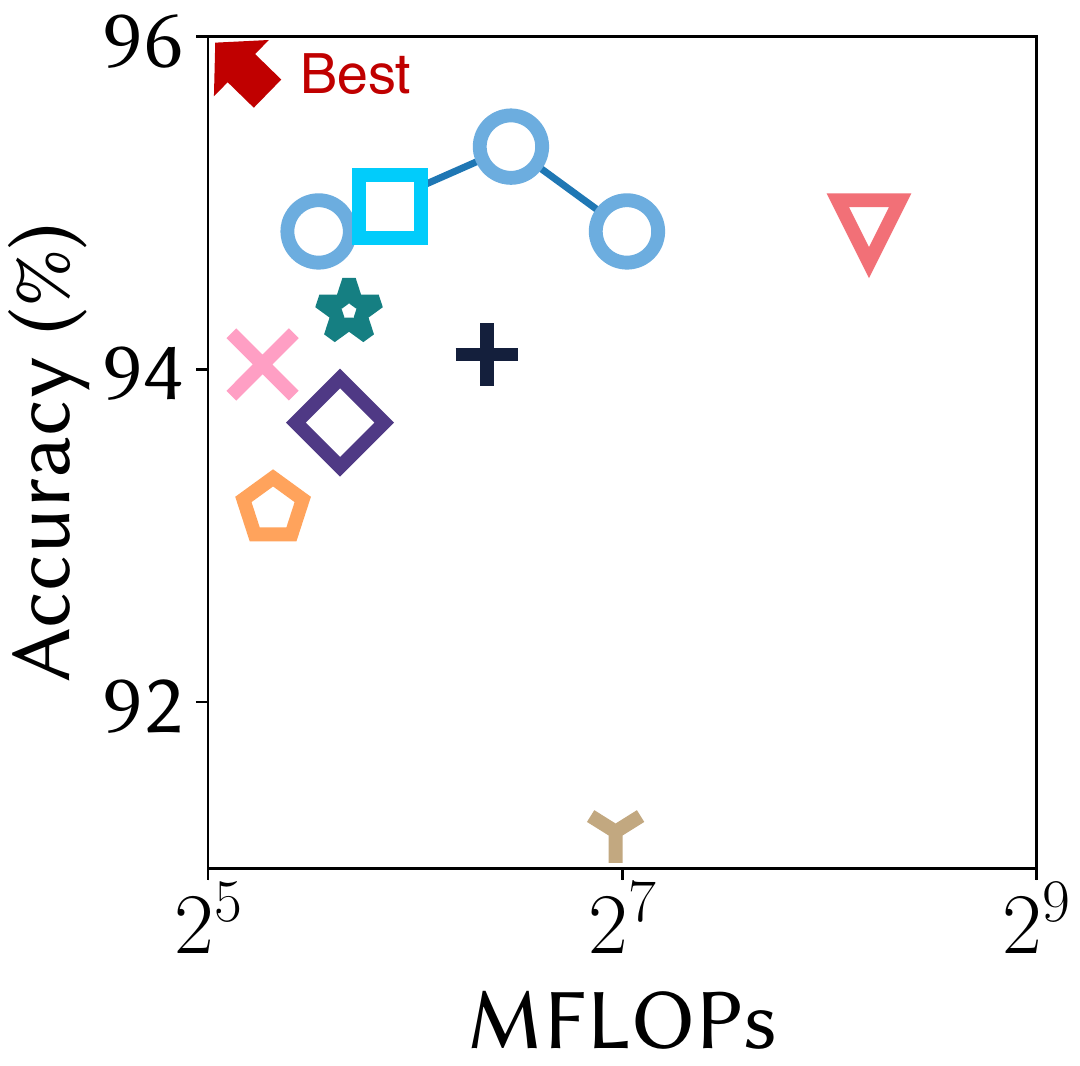}\label{fig:resnet_svhn_flops}}
	\subfloat[ResNet18-ImageNet] {\includegraphics[width=0.27\textwidth]{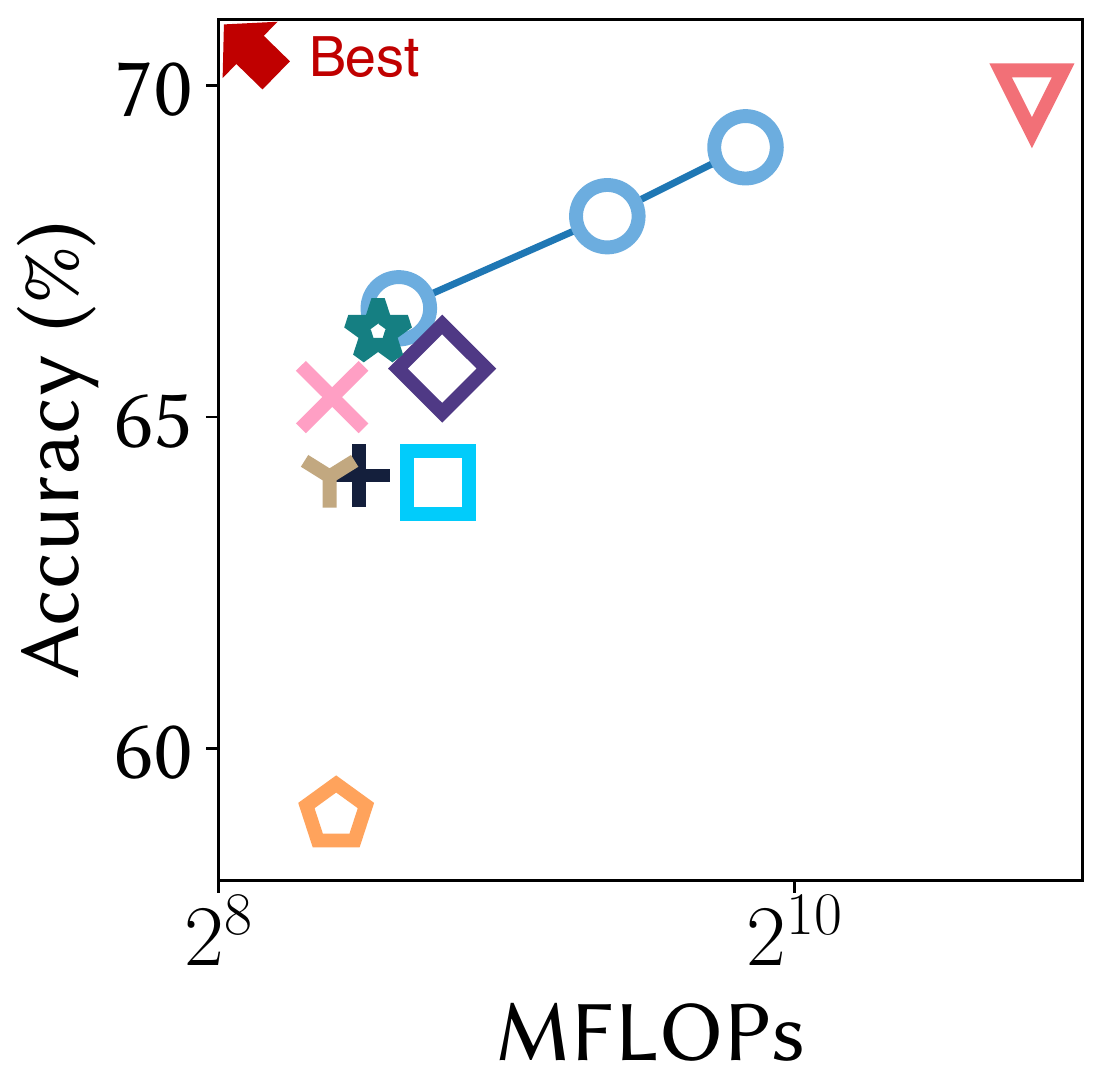}\label{fig:resnet_imagenet_flops}}
	\caption{Accuracy w.r.t. number of parameters on different models and datasets.
The three blue circles correspond to rank-1, 2, and 3 \method (from left to right order), respectively.
\method provides the best accuracy for a given number of parameters.
}
\label{fig:acc_flops}
\end{figure*}

\setlength{\tabcolsep}{2.1pt}
\begin{table}[t]
\centering
\caption{
\method and \method-branch give the best accuracy for similar number of parameters and FLOPs.
Bold font indicates the best accuracy among competing compression methods.
}
\label{tab:perf_FALCON}
\adjustbox{width=\textwidth}{
\begin{tabular}{l|rrr|l|rrr}
\toprule
\multirow{2}{*}{ConvType}       & \multicolumn{3}{c|}{VGG19-CIFAR100} & \multirow{2}{*}{ConvType}  & \multicolumn{3}{c}{ResNet34-CIFAR100} \\
  				& Accuracy 				& \# param    	& \# FLOPs & &  Accuracy 	& \# param	& \# FLOPs  	\\
\midrule
StConv           				& 72.10\%				& 20.35M      	& 398.75M	& StConv           				& 73.94\%				& 21.34M      	& 292.57M	\\
\midrule
\method           				& 72.01\%  				& 2.61M  		& 47.28M 	& \method           				& \textbf{71.83\%}  	& 2.67M  		& 46.38M	\\
\method without fitting          				& 70.80\%  				& 2.61M  		& 47.28M 	& \method without fitting		& 71.80\%  	& 2.67M  		& 46.38M			\\
\method-branch 1.75$\times$ 	& \textbf{73.05\%}		& 2.68M  		& 54.21M	& \method-branch 1.625$\times$    & 70.26\%				& 2.54M  		& 58.93M 	\\
GEP-in					& 68.29\%		& 2.61M 	& 46.46M	 & GEP-in 					& 66.88\%		& 2.67M	& 38.45M	\\
DPConv              				& 68.18\% 				& 2.61M      	& 48.07M	& DPConv              				& 66.30\%  				& 2.67M       	& 38.45M	\\
PDPConv-0.5 				& 72.50\%  				& 2.71M       	& 51.85M & PDPConv-0.5 				& 65.00\%  				& 2.59M       	& 39.83M		\\
GDGConv 2$\times$(g=2) 		& 72.73\%		  		& 2.79M 		& 46.71M 	& GDGConv 2$\times$(g=2) 	  	& 68.97\% 				& 3.17M 		& 49.88M		\\
PDPConv-split 1.375$\times$ 	& 72.32\%  				& 2.91M  		& 58.29M	& PDPConv-split 1.375$\times$  	& 67.38\%  				& 3.04M   	 	& 51.36M		\\
\midrule
\tucker & $57.32$\% & $4.10$M  & $84.89$M  & \tucker & $58.10$\% & $2.66$M & $54.03$M \\
\cp & $61.08$\% & $2.63$M & $68.71$M  & \cp & $65.82$\% & $2.66$M & $42.58$M   \\
\bottomrule
\end{tabular}}
\adjustbox{width=\textwidth}{
\begin{tabular}{l|rrr|l|rrr}
\toprule
\multirow{2}{*}{ConvType}       & \multicolumn{3}{c|}{VGG19-SVHN} & \multirow{2}{*}{ConvType}  & \multicolumn{3}{c}{ResNet34-SVHN} \\
  				& Accuracy 				& \# param    	& \# FLOPs & &  Accuracy 	& \# param	& \# FLOPs  	\\
\midrule
StConv           				& 95.28\%  				& 20.30M      	& 398.70M	& StConv         					& 94.83\%  				& 21.29M      	& 292.52M	\\
\midrule
\method           				& \textbf{95.45\%}  	& 2.56M  		& 47.23M & \method           				& 94.83\%				& 2.63M  		& 46.33M		\\
\method without fitting          				& {94.51\%}  	& 2.56M  		& 47.23M  & \method without fitting          				& 94.75\%				& 2.63M  		& 46.33M		\\
\method-branch 1.75$\times$     & 94.56\%				& 2.64M  		& 54.17M & \method-branch 1.625$\times$    & \textbf{94.98\%}		& 2.47M  		& 58.86M		\\
GEP-in 					& 94.99\%		& 2.56M	& 46.41M	 & GEP-in 					& 94.06\%		& 2.62M	& 38.41M	\\
DPConv              				& 94.37\%  				& 2.56M      	& 48.02M	& DPConv              				& 94.03\%	  			& 2.62M       	& 38.41M		\\
PDPConv-0.5 				& 93.28\%  				& 2.67M       	& 51.80M	& PDPConv-0.5 				& 93.16\%  				& 2.55M       	& 39.78M	\\
GDGConv 2$\times$(g=2) 	  	& 93.15\% 				& 2.74M     	& 46.66M	& GDGConv 2$\times$(g=2) 	  	& 93.68\% 				& 3.08M    		& 49.78M	\\
PDPConv-split 1.375$\times$ 	& 94.36\%  				& 2.86M       	& 58.24M	& PDPConv-split 1.375$\times$ 	& 94.35\%  				& 2.98M 		& 51.30M		\\
\midrule
\tucker & $95.30$\% & $2.35$M  & $69.28$M & \tucker & $91.22$\% & $2.88$M  & $125.23$M\\
\cp & $93.13$\% & $2.54$M  & $68.15$M & \cp & $94.09$\% & $2.63$M  & $81.34$M \\

\bottomrule
\end{tabular}}
\adjustbox{width=\textwidth}{
\begin{tabular}{l|rrrr|l|rrrr}
\toprule
\multirow{2}{*}{ConvType}       & \multicolumn{4}{c|}{VGG16\_BN-ImageNet} & \multirow{2}{*}{ConvType}  & \multicolumn{4}{c}{ResNet18-ImageNet} \\
  				& Top-1 & Top-5 & \# param    	& \# FLOPs & &  Top-1 & Top-5 	& \# param	& \# FLOPs  	\\
\midrule
StConv
 & 73.37\% 	& 91.50\% 	& 138.37M 	& 15484.82M	& StConv
 & 69.76\%  	& 89.08\% 	& 11.69M 	& 1814.07M	\\
\midrule
\method 					
 & \textbf{71.93}\% 	& \textbf{90.57\%} 	& 125.33M 	& 1950.75M & \method
 & \textbf{66.64\%}  	& 87.09\%	& 1.97M  	& 395.40M	\\
\method without fitting 					
 & 71.65\% 	& {90.47\%} 	& 125.33M 	& 1950.75M	&
 \method without fitting
 & {66.19\%}  	& 86.86\%	& 1.97M  	& 395.40M\\
\method-branch 1.5$\times$ 					
& 68.24\% 	& 88.51\% 	& 125.30M 	& 1898.39M	& \method-branch 1.375$\times$ 					
& 64.01\% 	& 85.16\% 	& 1.91M 	& 434.44M\\
GEP-in               		
 & 70.98\%  	& 90.19\%	& 125.33M	& 1910.56M	& GEP-in
 & 66.21\%  	& 86.93\%	& 1.96M    	& 336.81M	\\
DPConv              			
 & 70.34\%  	& 89.71\%	& 125.33M	& 1989.49M & DPConv
 & 65.30\%  	& 86.30\%	& 1.96M    	& 336.81M		\\
PDPConv-0.5 		 	
 & 67.80\%  	& 87.90\%	& 125.44M 	& 2180.49M & PDPConv-0.5
 & 58.99\% 		& 81.55\%	& 1.90M     & 340.06M	\\
GDGConv 2$\times$(g=2)
 & 70.40\%		& 89.84\% 	& 125.77M 	& 2014.73M & GDGConv 2$\times$(g=2) 	
 & 65.73\% 		& 86.75\%	& 2.22M		& 438.89M	\\
PDPConv-split 1.25$\times$
 & 71.34\%		& 90.34\% 	& 125.57M 	& 2180.65M & PDPConv-split 1.1875$\times$
 & 66.29\% 		& \textbf{87.32\%}	& 2.01M		& 376.15M	\\
 \midrule
\tucker & $70.23$\% & $89.61$\% & $126.30$M &  $4214.26$M  & \tucker & $64.10$\% & $85.58$\% & $2.01$M & $334.73$M \\
\cp & $64.40$\% & $85.58$\% & $124.79$M &  $1796.79$M & \cp & $64.10$\% & $85.51$\% & $1.97$M  & $359.80$M \\
\bottomrule
\end{tabular}}
\end{table}

\subsection{Performance (Q1)}
\label{subsec:performance}
{We evaluate the performance of \method in terms of accuracy, compression rate, and the amount of computation.}

\textbf{Accuracy vs. Compression.}
We evaluate the accuracy and the compression rate of \method and competitors.
Table~\ref{tab:perf_FALCON} shows the results on four image datasets.
Note that \method or \method-branch provides the highest accuracy
in most cases (5 out of 6) 
while using a similar or smaller number of parameters than competitors.
Specifically, \method and \method-branch achieve up to $8.61\times$ compression rates with minimizing accuracy drop compared to that of the standard convolution (StConv).
Figure~\ref{fig:acc_size} shows the tradeoff between accuracy and the number of parameters.
Note that \method and \method-branch show the best tradeoff (closest to the ``best" point) between accuracy and compression rate, giving the highest accuracy with similar compression rates.

\textbf{Accuracy vs. Computation.}
We evaluate the accuracy and the amount of computation of \method and competitors.
We use the number of multiply-add floating point operations (FLOPs) needed for forwarding one image to a model as the metric of computation.
Table~\ref{tab:perf_FALCON} also shows
the accuracies and the number of FLOPs of methods on four image datasets.
Note that \method or \method-branch provides the highest accuracy
in most cases 
while using similar FLOPs as competitors do.
Compared to StConv, \method and \method-branch achieve up to $8.44\times$ FLOPs reduction across different models on different datasets.
%
%
Figure~\ref{fig:acc_flops} shows the tradeoff between accuracy and the number of FLOPs.
Note that \method and \method-branch show the best tradeoff (closest to the "best" point) between accuracy and computation, giving the highest accuracy with a similar number of FLOPs.


\subsection{Ablation study (Q2)}
\label{subsec:ablation}
We perform an ablation study on two components:
1) fitting the depthwise and the pointwise kernels to the trained standard convolution kernel,
and
2) the order of depthwise and pointwise convolutions.
From Table~\ref{tab:perf_FALCON}, we have two observations.
First, fitting kernels using GEP gives better accuracy:
\method and GEP-in give better accuracy than \method without fitting and DPConv, respectively.
Second, with a similar number of parameters,
\method and \method without fitting always result in better accuracy than GEP-in and DPConv, respectively; it shows that applying pointwise convolution first and then depthwise convolution is better than the other way.
%
%
These observations show
that our main ideas of fitting kernels using GEP and careful ordering of convolutions are key factors for the superior performance.

\setlength{\tabcolsep}{2.1pt}
\begin{table}[t]
\centering
\caption{
Rank-$k$ \method further improves accuracy while sacrificing a bit of compression and computation reduction.
}
\label{tab:rank_perf}
\adjustbox{width=\textwidth}{
\begin{tabular}{l|rrrrr|rrrrr}
\toprule
\multirow{2}{*}{ConvType}       & \multicolumn{5}{c|}{VGG19-CIFAR100} & \multicolumn{5}{c}{ResNet34-CIFAR100} \\
  				& Accuracy 				& \# param  &   	& \# FLOPs &  &  Accuracy 	& \# param & 	& \# FLOPs &  	\\
\midrule
StConv    & $72.10$\%  & $20.35$M  & 			   & $398.75$M  		&	& $73.94$\%  & $21.34$M       & 			     & $292.57$M		 	&		\\
\midrule
\method-k1 & $72.01$\%  & $2.61$M & ($7.80\times$)  & $47.28$M		& ($8.43\times$) & $71.83$\%  & $2.67$M & ($7.99\times$)  & $46.38$M		& ($6.31\times$)\\
\method-k2 & $73.71$\%  & $4.88$M & ($4.17\times$)  & $94.24$M		& ($4.23\times$) &  $72.94$\%  & $5.08$M & ($4.20\times$)  & $88.25$M		& ($3.32\times$)			\\
\method-k3 & $73.41$\% & $7.16$M & ($2.84\times$)  & $141.21$M		& ($2.82\times$) &  $72.85$\%  & $7.49$M & ($2.85\times$)  & $130.13$M		& ($2.25\times$)	\\
\bottomrule
\end{tabular}}
\adjustbox{width=\textwidth}{
\begin{tabular}{l|rrrrr|rrrrr}
\toprule
\multirow{2}{*}{ConvType}       & \multicolumn{5}{c|}{VGG19-SVHN} & \multicolumn{5}{c}{ResNet34-SVHN} \\
  				& Accuracy 				& \# param  &   	& \# FLOPs &  &  Accuracy 	& \# param & 	& \# FLOPs &  	\\
\midrule
StConv    & 95.28\%  & 20.30M & 			     & 398.70M		&	& 94.83\%  & 21.29M       &   			     & 292.52M			&				\\
\midrule
\method-k1 & 95.45\%  & 2.56M & (7.93$\times$)  & 47.23M		& (8.44$\times$)	& 94.83\%  & 2.63M & (8.10$\times$)  & 46.33M 		& (6.31$\times$) 	\\
\method-k2 & 95.72\%  & 4.84M & (4.19$\times$)  & 94.20M		& (4.23$\times$)	& 95.34\%  & 5.04M & (4.22$\times$)  & 88.21M		& (3.32$\times$) 	\\
\method-k3 & 95.54\%  & 7.11M & (2.86$\times$)  & 141.16M		& (2.82$\times$) & 94.83\%  & 7.45M & (2.86$\times$)  & 130.08M		& (2.25$\times$)	\\
\bottomrule
\end{tabular}}
\adjustbox{width=\textwidth}{
\begin{tabular}{l|rrrrrr|rrrrrr}
\toprule
\multirow{2}{*}{ConvType}       & \multicolumn{6}{c|}{VGG16\_BN-ImageNet} & \multicolumn{6}{c}{ResNet18-ImageNet} \\
  				& Top-1 & Top-5 	& 	 \# param  &   	& \# FLOPs &  &  Top-1 & Top-5 	& \# param & 	& \# FLOPs &  	\\
\midrule
StConv 		& 73.37\% 	& 91.50\% 	& 138.37M 	& 	& 15484.82M 	&	& 69.76\%  	& 89.08\% 	& 11.69M 	& 	& 1814.07M 	&	\\
\midrule
\method-k1	& 71.93\% 	& 90.57\% 	& 125.33M 	& (1.10$\times$)	& 1950.75M 	& (7.94$\times$)& 66.64\%  	& 87.09\%	& 1.97M 		& (5.93$\times$) 	& 395.40M		& (4.59$\times$) \\
\method-k2	& 72.88\% 	& 91.19\% 	& 127.00M 	& (1.09$\times$)	& 3777.86M 	& (4.10$\times$) & 68.03\%  	& 88.26\%	& 3.22M  	& (3.63$\times$) 	& 653.00M		& (2.78$\times$)	\\
\method-k3	& 73.24\% 	& 91.54\%	& 128.68M	& (1.08$\times$)& 5604.97M 	& (2.76$\times$) & 69.07\%  	& 88.56\%	& 4.48M  	& (2.61$\times$) 	& 910.61M		& (1.99$\times$)	\\
\bottomrule
\end{tabular}}
\end{table}

\subsection{Rank-k \method (Q3)}
\label{subsec:rank_perf}

We evaluate the performance of rank-$k$ \method by increasing $k$ and monitoring the changes in the numbers of parameters
and FLOPs.
In Table~\ref{tab:rank_perf}, we observe three trends as $k$ increases:
1) the accuracy becomes higher, 
2) the number of parameters increases,
and
3) the number of floating point operations (FLOPs) increases.
Although the $k$ that provides the best tradeoff of rank and compression/computation reduction varies,
rank-$k$ \method improves the accuracy of \method in all cases.
Especially, we note that rank-$k$ \method often results in even higher accuracy than the standard convolution, while using smaller number of parameters and FLOPs.
For example, rank-$3$ \method applied to VGG19 on CIFAR100 dataset shows $1.31$ percentage point higher accuracy compared to the standard convolution, with $2.8\times$ smaller number of parameters and $2.8\times$ smaller number of FLOPs.
Thus, rank-$k$ \method is a versatile method to further improve the accuracy of \method while sacrificing a bit of compression and computation reduction.

\subsection{\method vs. Tensor decomposition (Q4)}
\label{subsec:td}

Since \method can be thought of as a variant of tensor decomposition,
we compare \method with CP and Tucker decomposition methods.
We compress standard convolution kernels using CP and Tucker,
and set the number of parameters and FLOPs to be approximately equal to those of \method.
We have two main observations from
Figure~\ref{fig:acc_size}, Figure~\ref{fig:acc_flops}, and
Table~\ref{tab:perf_FALCON}.
First, \method outperforms the two tensor decomposition methods for all datasets.
Second, \method has a better model capacity than those of CP and Tucker, for a given number of parameters and FLOPs.
CP and Tucker give poor accuracy compared to the standard convolution for CIFAR100,
while the accuracy gap becomes smaller for CIFAR10 and SVHN.
Note that CIFAR100 is a more complex dataset with more classes compared to those of CIFAR10 and SVHN, while the numbers of training instances are similar;
CP and Tucker lack model capacity to handle CIFAR100.
On the other hand, \method consistently shows the best accuracy due to its rich model capacity for handling convolutions.

\section{Related Work}
Over the past several years,
a lot of studies focused on compressing and accelerating deep neural networks to reduce model size, running time, and energy consumption.

Deep neural networks are often over-parameterized.
Weight-sharing (\cite{SoftWeight-Sharing,HashedNet,DeepCompression,TowardsTheLimitOfNetworkQuantization,VQ,FD,cho2020peakd}) is a common compression method which stores only assignments and centroids of weights.
While using the model, weights are loaded according to assignments and centroids.
Pruning (\cite{PruningWeightsAndConnections,PruningFilters,GuoO020,FrankleC19,ZhuangTZLGWHZ18}) aims at removing useless weights or setting them to zero.
Although weight-sharing and pruning can significantly reduce the model size,
they are not efficient in reducing the amount of computation.
Quantizing (\cite{BinaryConnect,BNN,LAB,TTQ,EVol}) the model into binary or ternary weights reduces model size and computation simultaneously:
replacing arithmetic operations with bit-wise operations remarkably accelerates the model.
Layer-wise approaches are also employed to efficiently compress models.
A typical example of such approaches is low-rank approximation (\cite{TuckerDecomposition,CPDecomposition,TTDecomposition,YooCKK19});
it treats the weights as a tensor and uses general tensor approximation methods like Tucker and CP to compress the tensor.
To reduce computation, approximation methods should be carefully chosen, since some of approximation methods may increase computation of the model.
A reinforcement learning based framework is also proposed in \cite{lee2020auber}. 

{A recent major trend is to design a brand new architecture that is small and efficient using depthwise separable convolution.}
Mobilenet (\cite{MobileNet}), MobilenetV2 (\cite{MobileNetV2}), Shufflenet (\cite{ShuffleNet}), and ShufflenetV2 (\cite{ShuffleNetV2}) are the most representative approaches, and
many works (\cite{Xception,MobileNetV3,WuDZWSWTVJK19,WanDZHTXWYXCVG20}) including them use depthwise and pointwise convolutions as building blocks for designing convolution layers.
However, they do not utilize the knowledge of a trained CNN model, and provide heuristic methods without precise understanding of the relation between depthwise separable convolution and standard convolution.
Our proposed \method gives a thorough interpretation of depthwise and pointwise convolutions, and exploits them into model compression,
giving the best accuracies with regard to compression and computation.

\section{Conclusion}

We propose \method, an accurate and lightweight convolution method for CNN.
By interpreting existing convolution methods based on depthwise separable convolution using GEP operation,
\method and its general version rank-$k$ \method provide accurate and efficient compression on CNN.
We also propose \method-branch, a variant of \method integrated into a branch architecture of CNN for model compression.
Extensive experiments show that
\method and its variants give the best accuracy for a given number of parameter or computation,
outperforming other compression models based on depthwise separable convolution and tensor decompositions.
Compared to the standard convolution,
\method and \method-branch give up to 8.61$\times$ compression and 8.44$\times$ computation reduction while giving similar accuracy.
We also show that rank-$k$ \method provides even better accuracy often outperforming the standard convolution,
while using smaller numbers of parameters and computations.
Future works include extending \method for other deep learning models beyond CNN.

\bibliography{dmlab}
\bibliographystyle{iclr2021_conference}


\end{document}